\documentclass[10pt,twocolumn,letterpaper]{article}

\usepackage[pagenumbers]{iccv} 

\usepackage{amsmath}
\usepackage{amssymb}
\usepackage{amsthm}
\usepackage{mathtools}
\usepackage{enumitem}
\usepackage{algorithm}
\usepackage{algpseudocode}
\usepackage{xcolor}
\usepackage{caption}
\usepackage{multirow}
\usepackage{array}
\usepackage{bm}
\usepackage{booktabs}

\DeclareMathOperator*{\argmin}{arg\,min}

\newcommand{\calD}{\mathcal{D}}

\newcommand{\calF}{\mathcal{F}}

\newcommand{\calM}{\mathcal{M}}

\newcommand{\calQ}{\mathcal{Q}}

\newcommand{\R}{\mathbb R}
\newcommand{\E}{\mathbb{E}}

\def\norm#1{\left\|#1\right\|}
\def\inner#1{\left\langle#1\right\rangle}

\newcommand{\vect}{{\rm Vec}}
\newcommand{\dist}{{\rm dist}}

\newtheorem{definition}{Definition}[section]
\newtheorem{theorem}{Theorem}[section]

\newtheorem{proposition}{Proposition}[section]

\newtheorem{assumption}{Assumption}[section]

\definecolor{iccvblue}{rgb}{0.21,0.49,0.74}
\usepackage[pagebackref,breaklinks,colorlinks,allcolors=iccvblue]{hyperref}

\title{Memory-Efficient 4-bit Preconditioned Stochastic Optimization}

\author{
Jingyang Li$^{1}$ \hspace{2em} Kuangyu Ding$^{1}$ \hspace{2em} Kim-Chuan Toh$^{1}$ \hspace{2em} Pan Zhou$^{2}$ \\
$^{1}$ National University of Singapore\hspace{4em} $^{2}$ Singapore Management University\\
{\tt\small $^{1}$\{li\_jingyang,kuangyud\}@u.nus.edu} \hspace{0.5em} {\tt\small mattohkc@nus.edu.sg} \hspace{3em} {\tt\small $^{2}$panzhou@smu.edu.sg}
}

\begin{document}
\maketitle
\begin{abstract}
Preconditioned stochastic optimization algorithms, exemplified by Shampoo, outperform first-order optimizers by offering theoretical convergence benefits and practical gains in large-scale neural network training. However, they incur substantial memory overhead due to the storage demands of non-diagonal preconditioning matrices. To address this, we introduce 4-bit quantization for Shampoo's preconditioners. We introduce two key methods: First, we apply Cholesky decomposition followed by quantization of the Cholesky factors, reducing memory usage by leveraging their lower triangular structure while better preserving spectral properties to minimize information loss. To our knowledge, this is the first quantization approach applied to Cholesky factors of preconditioners. Second, we incorporate error feedback in the quantization process, efficiently storing Cholesky factor and error state in the lower and upper triangular parts of the same matrix. Through extensive experiments, we demonstrate that combining Cholesky quantization with error feedback enhances memory efficiency and algorithm performance in large-scale deep-learning tasks. Theoretically, we also provide convergence proofs for quantized Shampoo under both smooth and non-smooth stochastic optimization settings.
\end{abstract}

\section{Introduction}
\vspace{-0.5em}
Deep learning has achieved significant advancements across numerous fields in recent years, including language modeling \cite{brown2020language,touvron2023llama}, computer vision \cite{dosovitskiy2020image}, and multi-modality \cite{radford2021learning}. These advancements are primarily driven by the scaling of model size, dataset volume, and computational power, as outlined in scaling laws that demonstrate the impact of increased resources on model performance \cite{kaplan2020scaling,hoffmann2022training}. This trend of scaling has further extended into specialized domains such as finance \cite{wu2023bloomberggpt}, material science \cite{xie2018crystal}, and healthcare \cite{lee2020biobert}.

Along with the size growth of large-scale models, stochastic gradient descent (SGD) has become a widely adopted method for training thanks to its efficiency and simplicity \cite{robbins1951stochastic, sutskever2013importance, he2016deep}. However, adaptive gradient methods, e.g., Adagrad \cite{duchi2011adaptive}, Adam \cite{kingma2014adam}, and AdamW \cite{loshchilov2019decoupled}, apply a diagonal preconditioning to the gradient, which enables faster convergence than SGD \cite{duchi2011adaptive,zhang2020adaptive}. These adaptive methods have demonstrated empirical advantages in various applications \cite{dosovitskiy2020image, xie2024adan} and are now the standard optimizers for training large-scale neural networks.

\begin{figure}[t]
	\centering
	\includegraphics[width=1.\linewidth]{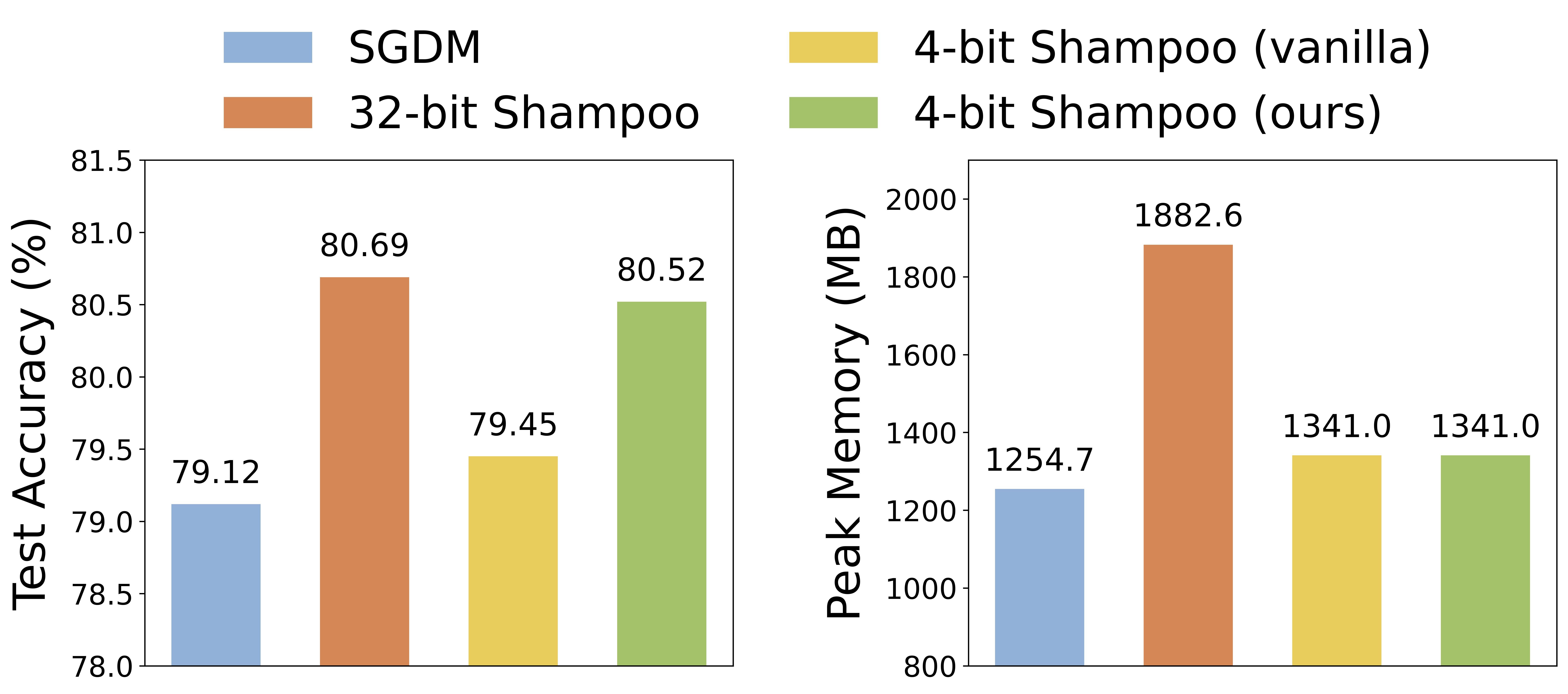}
    \vspace{-1.8em}
	\caption{Comparison of test accuracy and peak memory usage for training ResNet-34 on CIFAR-100 dataset.}
	\label{fig:compare}
    \vspace{-0.6em}
\end{figure}
Building on adaptive gradient methods, full-matrix preconditioned gradient methods offer theoretically superior convergence by capturing richer correlations among parameters \cite{duchi2011adaptive}. Despite these theoretical advantages, however, the memory overhead associated with non-diagonal matrices poses a significant challenge for large-scale neural networks, which can contain millions of parameters \cite{he2016deep, dosovitskiy2020image, liu2021swin}. To address this, a range of efficient preconditioned gradient methods, such as K-FAC \cite{martens2015optimizing}, Shampoo \cite{gupta2018shampoo}, K-BFGS \cite{goldfarb2020practical}, and AdaBK \cite{yong2023general}, aim to make full-matrix preconditioning computationally feasible by approximating the full-matrix preconditioner, e.g., block-diagonal precondition matrix. These algorithms have shown faster convergence rates in practice when compared to both SGD and adaptive gradient methods \cite{anil2020scalable, yong2023general, shi2023distributed}.  

Nevertheless, these efficient preconditioned methods still impose substantial memory costs that restrict their scalability in practical and large-scale model applications. As shown in \cref{fig:compare}, the peak memory usage of methods like Shampoo remains significantly higher than SGD.  While quantizing precondition matrices in these preconditioned methods from high-precision to low-precision, e.g., 32-bit to 4-bit, effectively reduces memory usage, it also inevitably introduces information loss, which, in turn, severely degrades the performance of these preconditioned methods. This is validated by \cref{fig:compare}:  compared with 32-bit Shampoo, 4-bit Shampoo enjoys much less memory cost but suffers from much worse performance. Therefore, for these efficient preconditioned methods, carefully designed strategies are essential to effectively compress precondition matrices without compromising optimization quality.

\noindent{\textbf{Contribution.}} We focus on Shampoo due to its simplicity, effectiveness, and popularity, aiming to enable efficient 4-bit quantization of preconditioners while maintaining the stability and efficiency of preconditioned gradient methods. Our main contributions are as follows:

\begin{itemize}
    \item We introduce \textbf{Cholesky quantization} to improve memory efficiency and stability. Instead of directly quantizing preconditioners, we apply Cholesky decomposition and quantize the Cholesky factors. This reduces storage by half while better preserving spectral properties, mitigating quantization-induced information loss. To the best of our knowledge, this is the first quantization approach applied to Cholesky factors of preconditioners. 
    
    \item We propose an \textbf{error feedback strategy} for Cholesky quantization to further reduce quantization error. Inspired by low-precision communication in distributed training \cite{seide20141, tang20211, xie2024loco}, we maintain a 4-bit error state that exponentially-moving averages past quantization errors for stable error estimation. This state compensates the Cholesky factor at each iteration, reducing information loss. Moreover, the triangular structure of the Cholesky factor allows efficient joint storage with its error state.
    
    \item We establish \textbf{convergence guarantees} for quantized Shampoo in both smooth and nonsmooth stochastic nonconvex optimization. In the smooth case, our 4-bit Shampoo achieves the optimal $\mathcal{O}(\frac{1}{\sqrt{T}})$ convergence rate. In the nonsmooth case (e.g., ReLU-based networks \cite{he2016deep}), we provide the first proof of global convergence for preconditioned gradient descent, showing convergence to stationary points under mild conditions.

    \item We develop \textbf{4-bit Shampoo} using these techniques and evaluate it on image classification with convolutional neural networks (CNNs) and vision transformers (ViTs). It outperforms vanilla 4-bit Shampoo, and significantly reduces memory usage compared to 32-bit Shampoo while maintaining comparable test performance, enabling larger models to be trained within existing resource constraints.
\end{itemize}

\vspace{-0.5em}
\section{Related Work}
\vspace{-0.5em}
\noindent \textbf{Preconditioned Stochastic Optimization.}
Adaptive gradient methods are the most widely used preconditioned gradient methods in training neural networks, with Adagrad \cite{duchi2011adaptive}, RMSProp \cite{tieleman2012lecture}, and Adam \cite{kingma2014adam} being notable examples. They use diagonal preconditioners to rescale the gradients, been shown to improve convergence in stochastic settings. Preconditioned gradient methods with non-diagonal preconditioners offer faster convergence in theory \cite{duchi2011adaptive}, and are widely explored recently due to faster convergence than adaptive gradient methods in practice \cite{martens2015optimizing,gupta2018shampoo,goldfarb2020practical,yong2023general}. Among them, Shampoo \cite{gupta2018shampoo} receives extensive concern for its simplicity and effectiveness \cite{morwani2024new,vyas2024soap,wang20244}, and it has been developed for large-scale distributed training \cite{anil2020scalable,shi2023distributed}.

\noindent \textbf{Quantization for Optimizers.} Quantization has been widely used for gradient compression to enable efficient communication in large-scale optimization, particularly for distributed training \cite{alistarh2017qsgd,wen2017terngrad,vogels2019powersgd}. Recent works have extended quantization to optimizer states—such as the momentum or second-moment estimates used by adaptive optimizers like Adam—to decrease peak memory usage during neural network training \cite{dettmers20218,li2024memory}. Despite its computational efficiency, quantization incurs information loss, which can degrade algorithmic performance. To address this, ongoing research explores techniques such as error feedback compensation to mitigate these effects and improve robustness \cite{seide20141,richtarik2021ef21}.

\vspace{-0.5em}
\section{Preliminaries}
\vspace{-0.7em}
Here we introduce practical Shampoo from~\cite{anil2020scalable}, and linear-square (linear-2) quantization~\cite{dettmers20218} to compress the preconditioning matrices in our algorithm.

\noindent \textbf{Notations.} Let \(\|A\|_F = \sqrt{\sum_{ij} A_{ij}^2}\) denote the Frobenius norm of a matrix \(A\), and \(\langle A, B \rangle = \sum_{ij} A_{ij}B_{ij}\) its inner product. The Kronecker product of \(A\) and \(B\) is denoted by \(A \otimes B\). For a symmetric matrix \(H\), \(\lambda_{\max}(H)\) and \(\lambda_{\min}(H)\) represent its maximum and minimum eigenvalues, respectively. For square symmetric matrices \(A\) and \(B\), we write \(A \preceq B\) if \(B - A\) is positive semidefinite (PSD). Quantization and dequantization operations denoted by \(\mathcal{Q}\) and \(\mathcal{D}\).

\vspace{-0.3em}
\subsection{Practical Shampoo} \label{sec-shampoo}
\vspace{-0.5em}
When minimizing a nonconvex stochastic objective:
\begin{equation}\label{eq:loss}
\setlength{\abovedisplayskip}{1pt}
\setlength{\belowdisplayskip}{1pt}
\setlength{\abovedisplayshortskip}{1pt}
\setlength{\belowdisplayshortskip}{1pt}
F(W)\coloneqq \mathbb{E}_{\xi \sim \Xi} [F(W, \xi)],
\end{equation}
where \(W\in \mathbb{R}^{m \times n}\) is the parameter of the learning model, and data \(\xi\) is drawn from an unknown distribution \(\Xi\). At each iteration, we sample a mini-batch of data points to compute the stochastic gradient \(G \in \mathbb{R}^{m \times n}\), and use this stochastic gradient to update the model parameter \(W\).

To accelerate convergence, Shampoo preconditions the stochastic gradient used in first-order optimizers. Specifically, at iteration \(k\), it updates the preconditioning states \(L_k\) and \(R_k\) with stochastic gradient \(G_k\) for preconditioning:
\begin{equation}\label{shampoogradient}
\setlength{\abovedisplayskip}{1pt}
\setlength{\belowdisplayskip}{1pt}
\setlength{\abovedisplayshortskip}{1pt}
\setlength{\belowdisplayshortskip}{1pt}
\begin{cases}
	L_k = \beta L_{k-1} + (1-\beta) G_k G_k^T,\\
	R_k = \beta R_{k-1} + (1-\beta) G_k^T G_k, \\
	\hat{G}_k = L_k^{-1/4} G_k R_k^{-1/4},
\end{cases}
\end{equation}
where \(\beta \in (0,1)\), and the \(1/4\)-th root inverse is computed efficiently using the Schur-Newton algorithm \cite{guo2006schur}. 

Next, first-order base optimizer \(\calF\) like SGD can use the preconditioned gradient \(\Tilde{G}_k\) in Eq.~\eqref{shampoogradient} to replace vanilla  \(G_k\) for model update. For efficiency, Shampoo stores \((L_k,R_k,L_k^{-1/4},R_k^{-1/4})\), and  updates \((L_k,R_k)\) for every \(T_1\) iterations and \((L_k^{-1/4},R_k^{-1/4})\) every \(T_2\) iterations. See practical Shampoo algorithm in \cref{alg:32bit} of \cref{app-32bit}.
\vspace{-0.1em}
\subsection{Linear Square Quantization for Compression} \label{sec-linear-2}
\vspace{-0.5em} 
Quantization compresses tensors from high precision floating-point to low precision, reducing memory usage. Following \cite{dettmers20218,li2024memory}, we use block-wise quantization to mitigate outlier effects. Below, we introduce the quantization and dequantization processes, focusing on the two-dimensional tensor (matrix) case of Shampoo.

\noindent\textbf{Quantization.}  For a floating-point matrix \(X \in \mathbb{R}^{m \times n}\), we partition it into blocks of size \(B\times B\), resulting in \(P = \lceil m/B \rceil \times \lceil n/B \rceil\) blocks \(\{X_p\}_{p=1}^P\). In each block \(X_p\), a normalization factor \(N_p = \max \{|X_p|\}\) scales elements to \([-1, 1]\) via \(\Bar{X}_p = X_p / N_p\). Each element \(\bar{x}_p\) in \(\Bar{X}_p\) is then quantized to a \(b\)-bit integer using a quantization mapping \(\calM: [0, 2^b-1] \cap \mathbb{Z} \to [-1, 1]\), calculated by:
\begin{equation}
\setlength{\abovedisplayskip}{1pt}
\setlength{\belowdisplayskip}{1pt}
\setlength{\abovedisplayshortskip}{1pt}
\setlength{\belowdisplayshortskip}{1pt}
q_p = \argmin_{j \in [0,2^b - 1] \cap \mathbb{Z}} |\bar{x}_p - \calM (j)|.
\end{equation}
Common quantization mappings include linear, dynamic, and quantile mappings  \cite{dettmers20218,li2024memory,wang20244}. Here we  use a linear-2 mapping for simplicity and efficiency when \(b=4\):
\begin{equation}
\setlength{\abovedisplayskip}{1pt}
\setlength{\belowdisplayskip}{1pt}
\setlength{\abovedisplayshortskip}{1pt}
\setlength{\belowdisplayshortskip}{1pt}
\calM(j) = 
\begin{cases} 
	- ( -1 + \frac{2j}{2^b - 1} )^2, & j < 2^{b-1} - 1, \\
	0, & j = 2^{b-1} - 1, \\
	( -1 + \frac{2j}{2^b - 1} )^2, & j > 2^{b-1} - 1,
\end{cases}
\end{equation}
where \(j \!\in\! \{ 0, 1, \ldots, 2^b - 1 \}\). This block-wise quantization can be efficiently executed in parallel on GPUs \cite{gholami2022survey,yao2022zeroquant}.

\noindent\textbf{Dequantization.}  Dequantization \(\calD\) reverses the quantization process. For each quantized block \(Q_p\), we map each element \(q_p\) back to \([-1, 1]\) via \(\bar{x}'_p = \calM(q_p)\) to obtain \(\Bar{X}'_p\). We then restore the original scale using \(N_p\), giving \(X'_p = \calD(Q_p) = N_p \Bar{X}'_p\). Like quantization, dequantization is parallelizable on GPUs.  

For block size \(B\times B \), it balances accuracy and memory cost: smaller blocks improve accuracy but increase the number of normalization factors, raising memory overhead.

\vspace{-0.1em}
\section{Memory-Efficient  Shampoo Via Compensated Cholesky Quantization}
\vspace{-0.3em}
We first present a direct quantization method to reduce the memory overhead of Shampoo's preconditioning matrices in Sec.~\ref{sec-vanilla}. Then, in Sec.~\ref{sec-cholesky}, we introduce a more memory-efficient Cholesky quantization approach that better preserves spectral properties to enhance vanilla quantization. Finally, in Sec.~\ref{CholeskyCompensation}, we propose a compensation strategy to mitigate information loss from Cholesky quantization.

\vspace{-0.1em}
\subsection{Quantization for Shampoo Compression} \label{sec-vanilla}
\vspace{-0.3em}
From \cref{sec-shampoo}, one knows that Shampoo requires storage of four preconditioning matrices \((L_k, R_k, L_k^{-1/4}, R_k^{-1/4})\), each sized \(d \times d\),  where $d$ denotes the model parameter dimension. This brings much additional GPU memory cost, and becomes even more pronounced when training modern neural networks, which are often extremely high-dimensional. So  reducing Shampoo’s memory overhead is essential for efficient and scalable  network training. 

A straightforward approach is to use  a quantizer $\calQ$, e.g., the linear-2 quantization  in \cref{sec-linear-2},  to compress the preconditioners in Shampoo for saving memory,  and then adopt a dequantizer $\calD$ to recover  them for subsequent usage.  Formally, at iteration \(k\),  we can compute  two low-precision  preconditioning states \((\Bar{L}_k,\Bar{R}_k)\) as 
\begin{equation} 
\setlength{\abovedisplayskip}{4pt}
\setlength{\belowdisplayskip}{4pt}
\setlength{\abovedisplayshortskip}{4pt}
\setlength{\belowdisplayshortskip}{4pt}\label{eq:ori_q_state}
\begin{aligned}
& L_k = \beta \calD(\Bar{L}_{k-1}) + (1-\beta) G_k G_k^T, \ \Bar{L}_k = \calQ(L_k), \\
& R_k = \beta \calD(\Bar{R}_{k-1}) + (1-\beta) G_k^T G_k, \ \Bar{R}_k = \calQ(R_k). \\
\end{aligned}
\end{equation}
In this work, we use 4-bit precision  for efficient storage.   
For \(\Bar{L}_k^{-1/4}, \Bar{R}_k^{-1/4}\), we update them as 
\begin{equation} 
\setlength{\abovedisplayskip}{4pt}
\setlength{\belowdisplayskip}{4pt}
\setlength{\abovedisplayshortskip}{4pt}
\setlength{\belowdisplayshortskip}{4pt}
\label{eq:ori_q_prec}  
\begin{aligned} 
& L_k = \calD(\Bar{L}_k), \  \Bar{L}_k^{-1/4} = \calQ((L_k + \lambda_{\max}^L \epsilon I_m)^{-1/4}), \\
& R_k = \calD(\Bar{R}_k), \  \Bar{R}_k^{-1/4} = \calQ( (R_k + \lambda_{\max}^R \epsilon I_n)^{-1/4}), 
\end{aligned}
\end{equation}
where, same as vanilla  Shampoo,  \(\lambda_{\max}^L \epsilon I_m\) and \(\lambda_{\max}^R \epsilon I_n\) provide numerical stability during the Schur-Newton iterations used to calculate the inverse \(1/4\)-th roots, in which \(\lambda_{\max}^L\), \(\lambda_{\max}^R\) are the maximal singular values of \(L_k,R_k\), and \(\epsilon\) is a small constant \cite{yong2023general}.

Accordingly, one can store 4-bit \((\Bar{L}_k,\Bar{R}_k,\) \( \Bar{L}_k^{-1/4},\) \( \Bar{R}_k^{-1/4} )\) instead of their original 32-bit versions, and  dequantize them for usage, e.g.,  dequantizing \((\Bar{L}_k^{-1/4}\),  \( \Bar{R}_k^{-1/4})\) to compute preconditioned gradient in Eq.~\eqref{shampoogradient}. 

Despite its simplicity, direct quantization of preconditioners as in \cref{eq:ori_q_state} and \cref{eq:ori_q_prec} can lead to performance degradation due to information loss, e.g., quantizing them from 32-bit to 4-bit precision. For instance, when training ViT-Small \cite{dosovitskiy2020image} on CIFAR-100 \cite{krizhevsky2009learning} with Shampoo using AdamW as the base optimizer, the 32-bit version Shampoo achieves 77.95\% test accuracy, substantially outperforming the 4-bit quantized Shampoo, which reaches only 74.56\%. 
Further experimental comparisons can be found in \cref{sec-experiment}.

\vspace{-0.15em}
\subsection{Efficient and Stable Cholesky Quantization} \label{sec-cholesky}
\vspace{-0.35em}

Here we introduce Cholesky quantization (CQ) to further improve memory efficiency and also stability of quantization in Sec.~\ref{sec-vanilla}.  
Instead of quantizing \(L_k\) and \(R_k\), we  apply Cholesky decomposition on \(L_k\) and \(R_k\), and quantize their corresponding  Cholesky factors as \(\Bar{C}_k^L\) and \(\Bar{C}_k^R\) which are lower triangular matrices and require much less  storage. Formally, at iteration \(k\), this process can be written as
\begin{equation} 
\setlength{\abovedisplayskip}{4pt}
\setlength{\belowdisplayskip}{4pt}
\setlength{\abovedisplayshortskip}{4pt}
\setlength{\belowdisplayshortskip}{4pt}
\label{eq:cd_q_state}
\begin{aligned}
    & L_{k-1} \!=\! \calD(\Bar{C}_{k-1}^L)\calD(\Bar{C}_{k-1}^L)^T \!, R_{k-1} \!=\! \calD(\Bar{C}_{k-1}^R)\calD(\Bar{C}_{k-1}^R)^T, \\
    & L_k \!= \!\beta L_{k-1} \!+\! (1\!-\!\beta) G_k G_k^T,   R_k \!=\! \beta R_{k-1} \!+\! (1\!-\!\beta) G_k^T G_k , \\
    &C_{k}^L\!=\!\texttt{Cholesky}( L_k \!+\!  \epsilon I), C_{k}^R\!=\!\texttt{Cholesky}( R_k \!+\! \epsilon I),
\end{aligned}
\end{equation}
where \(\texttt{Cholesky}(L_k + \epsilon I)\) computes a lower triangular matrix \(C_k^L\) such that \(C_k^L {C_k^L}^T = L_k + \epsilon I\). The small term \(\epsilon I\) is added for numerical stability, with \(\epsilon\) as small constant. Once \(C_k^L\) and \(C_k^R\) are computed, they are quantized as: 
\begin{equation}
\setlength{\abovedisplayskip}{4pt}
\setlength{\belowdisplayskip}{4pt}
\setlength{\abovedisplayshortskip}{4pt}
\setlength{\belowdisplayshortskip}{4pt}
\Bar{C}_k^L = \calQ(C_{k}^L ), \qquad \ \Bar{C}_k^R = \calQ(C_{k}^R).
\end{equation}
Accordingly, we can only store two quantized lower triangular matrices $\Bar{C}_k^L$ and $ \Bar{C}_k^R.$ Here we quantize the off-diagonal part of $\Bar{C}_k^L$ and $ \Bar{C}_k^R$ into 4-bit precision while retaining the diagonal elements for 32-bit. This approach is used because off-diagonal elements have less impact on numerical stability, allowing reduced precision with minimal accuracy loss. In contrast, diagonal elements are crucial for overall stability and accuracy, so keeping them in 32-bit helps prevent error accumulation in the factorization.  

Now we discuss two advantages of Cholesky quantization. Firstly, Cholesky factors are lower triangular matrices, requiring nearly half the GPU memory compared to storing full preconditioners, reducing peak memory usage. Secondly, the preconditioner \(L_k\) recovered from \(L_{k} = \calD(\Bar{C}_{k}^L)\calD(\Bar{C}_{k}^L)^T\) remains symmetric and positive definite (PD), better preserving spectral properties. Consequently, its inverse \(1/4\)-th root more closely approximates the original 32-bit preconditioner. To quantify this preservation, we consider the Frobenius norm relative error (NRE) and angle error (AE) between matrices \cite{wang20244}, given by 
\begin{equation}
\setlength{\abovedisplayskip}{5pt}
\setlength{\belowdisplayskip}{5pt}
\setlength{\abovedisplayshortskip}{5pt}
\setlength{\belowdisplayshortskip}{5pt}
\begin{aligned}
\text{NRE} & = \| A^{-1/4} - g(A)^{-1/4} \|_F/ \| A^{-1/4} \|_F, \\
\text{AE} & = \arccos \left( \frac{\langle A^{-1/4}, g(A)^{-1/4} \rangle }{\| A^{-1/4} \|_F \| g(A)^{-1/4}\|_F} \right),
\end{aligned}
\end{equation}
where \( g \) represents the combined effect of quantization and dequantization.
We evaluate these metrics using both synthetic PD matrices and preconditioners from 32-bit Shampoo training of VGG-19 on CIFAR-100. As shown in \cref{tab:nre_ae}, Cholesky quantization significantly reduces both NRE and AE, demonstrating its effectiveness in preserving spectral properties. See \cref{app-mat} for further details.

\begin{table}[h]
\vspace{-0.4em}
    \centering
    \caption{NRE and AE on synthetic and real preconditioners for vanilla quantization (VQ) and Cholesky quantization (CQ).}
    \label{tab:nre_ae}
    \vspace{-0.7em}
    \begin{small}
    \renewcommand{\arraystretch}{0.85}
    \setlength{\aboverulesep}{1pt}
    \begin{tabular}{l|cc|cc}
        \toprule
        Quantization & \multicolumn{2}{c|}{VQ} & \multicolumn{2}{c}{CQ} \\
        & NRE & AE & NRE & AE \\
        \midrule
        Synthetic & 46.141 & 27.187 & 9.188  & 9.204  \\
        Epoch 50  & 29.041 & 19.353 & 5.367  & 5.366  \\
        Epoch 100 & 25.712 & 18.505 & 4.852  & 4.853  \\
        Epoch 150 & 25.351 & 19.317 & 4.788  & 4.788  \\
        Epoch 200 & 34.908 & 20.795 & 6.152  & 6.154  \\
        \bottomrule
    \end{tabular}
    \end{small}
\vspace{-0.4em}
\end{table}

Finally, we analyze the computational efficiency of Cholesky quantization. While Cholesky decomposition has a complexity of \(\mathcal{O}(n_p^3)\) for a matrix of dimension \(n_p\), it is applied layer-wise with a capped preconditioner order of 1200 (\cref{app-exp_hyp}), keeping the cost manageable. Additionally, the matrix multiplications for computing the inverse \(1/4\)-th root and gradient preconditioning also have a complexity of \(\mathcal{O}(n_p^3)\). Experimental results in \cref{tab:imagenet,tab:llm} confirm that CQ incurs minimal computational overhead.

\subsection{Compensated Cholesky Quantization}\label{CholeskyCompensation}
\vspace{-0.5em}

\begin{figure}[t]
	\centering
	\includegraphics[width=0.43\linewidth]{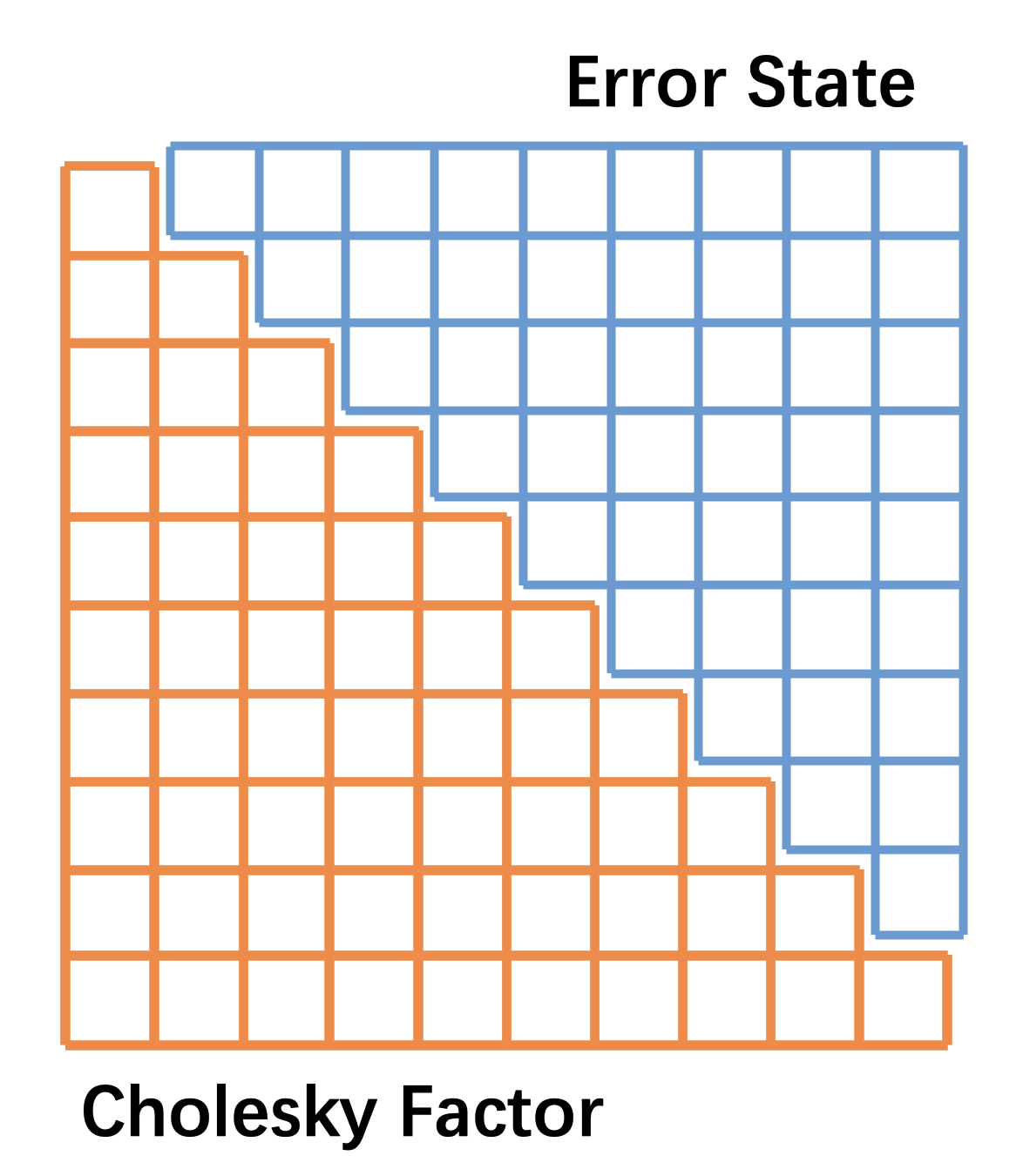}
        \vspace{-0.2em}
	\caption{Efficient storage for Cholesky factor and error state.}
    \vspace{-0.5em}
	\label{fig:ef}
\end{figure}

To mitigate the information loss from quantization, we introduce error feedback (EF) for Cholesky factors. Error feedback was original proposed to alleviate the information loss caused by gradient compression for communication in distributed training setting \cite{seide20141,richtarik2021ef21}. The key idea is to compensate for compression errors by adding them back into the gradients before compression in the next step. Practical adaptations of EF has also been explored in \cite{tang20211,xie2024loco} to combine EF with adaptive gradient methods for communication-efficient large-scale training. 

Different from previous work, our focus in this work is the compression of preconditioners of preconditioned gradient methods, and therefore our error feedback is conducted on the preconditioners. At each iteration, an additional low-precision (4-bit) error state, denoted as \(\Bar{E}_k^L\), is maintained to capture quantization error for the Cholesky factor \(\Bar{C}_k^L\). This error state is then used in the next iteration to enhance precision by compensating for potential quantization errors.

Specifically, at iteration \(k\), we first compute the Cholesky factors \(C_k^L\) and \(C_k^R\) following the standard steps in \cref{eq:cd_q_state}. Before quantizing, we apply error compensation as follows:
\begin{equation} 
\setlength{\abovedisplayskip}{5pt}
\setlength{\belowdisplayskip}{5pt}
\setlength{\abovedisplayshortskip}{5pt}
\setlength{\belowdisplayshortskip}{5pt}
\label{eq:cd_ef_state}
\begin{aligned}
& E_{k-1}^L = \calD(\Bar{E}_{k-1}^L),\qquad  \Bar{C}_k^L = \calQ(C_{k}^L + E_{k-1}^L), \\
& E_{k-1}^R = \calD(\Bar{E}_{k-1}^R),\qquad  \Bar{C}_k^R = \calQ(C_{k}^R + E_{k-1}^R).
\end{aligned}
\end{equation}
Next, we update the error states \(\Bar{E}_k^L\) and \(\Bar{E}_k^R\) using an exponential moving average to improve stability:
\begin{equation}
\setlength{\abovedisplayskip}{5pt}
\setlength{\belowdisplayskip}{5pt}
\setlength{\abovedisplayshortskip}{5pt}
\setlength{\belowdisplayshortskip}{5pt}
\label{eq:cd_ef_error}
\begin{aligned}
& E_k^L = \beta_e E_{k-1}^L + (1-\beta_e) (C_{k}^L + E_{k-1}^L - \calD(\Bar{C}_k^L)), \\
& E_k^R = \beta_e E_{k-1}^R + (1-\beta_e) (C_{k}^R + E_{k-1}^R - \calD(\Bar{C}_k^R)),
\end{aligned}
\end{equation}
where \(\beta_e\) is the momentum parameter. Since the Cholesky factors \(C_k^L\) and \(C_k^R\) are lower triangular and quantization excludes diagonal elements, the error states \(E_k^L\) and \(E_k^R\) are also triangular with zero diagonals. This enables efficient storage, as each error state can be stored as the upper triangular part as illustrated in \cref{fig:ef}, incurring no additional memory overhead compared to vanilla 4-bit Shampoo.

Finally, we can compute the inverse $1/4$-th root of the preconditioners with stored Cholesky factors via
\begin{equation}
\setlength{\abovedisplayskip}{3pt}
\setlength{\belowdisplayskip}{3pt}
\setlength{\abovedisplayshortskip}{3pt}
\setlength{\belowdisplayshortskip}{3pt}
\label{eq:cd_root}
\begin{aligned}
& \hat{L}_k = \calQ( (\calD(\Bar{C}_k^L) \calD(\Bar{C}_k^L)^T + \lambda_{\max}^L \epsilon I_m)^{-1/4} ), \\
& \hat{R}_k = \calQ( (\calD(\Bar{C}_k^R) \calD(\Bar{C}_k^R)^T + \lambda_{\max}^R \epsilon I_n)^{-1/4} ).
\end{aligned}
\end{equation}

Next, with SGD as the base optimizer, the model parameters are updated with the preconditioned gradient: 
\begin{equation}
\setlength{\abovedisplayskip}{3pt}
\setlength{\belowdisplayskip}{3pt}
\setlength{\abovedisplayshortskip}{3pt}
\setlength{\belowdisplayshortskip}{3pt}
\label{SGD}
W_{k+1} = W_k - \eta_k \calD (\hat{L}_k) G_k \calD(\hat{R}_k),
\end{equation}
where \(\eta_k\) is the learning rate for iteration \(k\) that is often scaled by \(\|G_k\|_F / \|\hat{G}_k\|_F\) according to the grafting trick \cite{agarwal2020disentangling}. The preconditioned stochastic gradient \(\calD(\hat{L}_k) G_k \calD(\hat{R}_k)\) can also be fed into another first-order optimizer $ \mathcal{F}$, such as Adam, for model updates.  Accordingly, we have arrived at our compensated Cholesky quantization based Shampoo summarized in \cref{alg:4bit}.

\begin{algorithm}[t]
\caption{4-bit Shampoo via Compensated Cholesky Quantization} \label{alg:4bit}
\textbf{Input:} initial weight $W_0 \in \mathbb{R}^{m \times n}$, initial Cholesky factors $\Bar{C}_0^L = \sqrt{\epsilon} I_m$, $\Bar{C}_0^R = \sqrt{\epsilon} I_n$, quantization error states $\Bar{E}_0^L = \bm{0}$, $\Bar{E}_0^R = \bm{0}$, initial preconditioners $\hat{L}_0 = I_m$, $\hat{R}_0 = I_n$. Total training iterations $T$, interval of updating preconditioners $T_1$ and $T_2$, momentum parameter $\beta,\beta_e \in (0,1)$. First-order optimizer $\mathcal{F}$ with initial optimizer state $s_0$.\\
\textbf{Output:} final weight $W_T$.
\begin{algorithmic}[1]
\For{$k = 1, 2, \ldots, T$}
    \State Compute gradient $G_k = \nabla \mathcal{L}_k(W_k)$
    \If{$k \% T_1 \equiv 0$}
        \State Update Cholesky factors according to \cref{eq:cd_q_state}
        \State Conduct error compensation following \cref{eq:cd_ef_state}
        \State Update quantization error states as \cref{eq:cd_ef_error}
    \Else
        \State $\Bar{C}_k^L = \Bar{C}_{k-1}^L, \ \Bar{C}_k^R = \Bar{C}_{k-1}^R$
    \EndIf
    \If{$k \% T_2 \equiv 0$}
        \State Compute inverse $1/4$-th root of the preconditioners following \cref{eq:cd_root} 
    \Else
        \State $\hat{L}_k = \hat{L}_{k-1}$, \  $\hat{R}_k = \hat{R}_{k-1}$
    \EndIf
    \State$\hat{G}_k =  \calD(\hat{L}_k) G_k \calD(\hat{R}_k)$
    \State $W_k, s_k = \mathcal{F}(W_{k-1}, s_{k-1}, \hat{G}_k)$
\EndFor
\end{algorithmic}
\end{algorithm}
\setlength{\textfloatsep}{18pt}

\vspace{-0.8em}
\section{Theoretical Analysis}
\vspace{-0.6em}
Here we provide theoretical analysis of \cref{alg:4bit} with SGD base optimizer as an example. We first define
\begin{equation}
\setlength{\abovedisplayskip}{3pt}
\setlength{\belowdisplayskip}{3pt}
\setlength{\abovedisplayshortskip}{3pt}
\setlength{\belowdisplayshortskip}{3pt}
\begin{aligned}
 & x_k := \vect(W_k),\; g_k := \vect(G_k), \\
 & H_k := \calD(\hat{R}_k) \otimes \calD(\hat{L}_k),
 \end{aligned}
\end{equation}
where $\vect$ reshapes the matrix into a vector by concatenating the columns of the matrix.   
Then, we rewrite Shampoo with SGD as base optimizer in Eq.~\eqref{SGD} into an equivalent  vectorization formulation: 
\begin{equation}
\setlength{\abovedisplayskip}{2pt}
\setlength{\belowdisplayskip}{2pt}
\setlength{\abovedisplayshortskip}{2pt}
\setlength{\belowdisplayshortskip}{2pt}
\label{generalSGD}
	x_{k+1}=x_k-\eta_kH_kg_k.
\end{equation}
See this equivalent derivation in \cref{app-proof}. In the following, we analyze Shampoo with SGD as base optimizer in Eq.~\eqref{generalSGD} under different situations. 

\vspace{-0.1em}
\subsection{Smooth Nonconvex Training Loss}
\vspace{-0.3em}
Here we analyze the smooth nonconvex  $f$, which is defined according to loss function \cref{eq:loss} as
\begin{equation}
\setlength{\abovedisplayskip}{2pt}
\setlength{\belowdisplayskip}{2pt}
\setlength{\abovedisplayshortskip}{2pt}
\setlength{\belowdisplayshortskip}{2pt}
f(x) \coloneqq F(W) 
,
\end{equation}
where \(x=\vect(W)\) is the vectorized model parameter. To this end, we introduce the necessary assumptions.
\vspace{-0.5em}
\begin{assumption} 
\label{assumption:func_noise}
\textbf{a)}  Assume the training loss $f$ is $L$-Lipschitz smooth, i.e.,  $\norm{\nabla f(x)-\nabla f(y)}_2\leq L\norm{x-y}_2$.  \\
\textbf{b)}  Suppose the stochastic gradient $g_k$ is unbiased and its variance can be bounded:
$\E[g_k]=\nabla f(x_k)$ and $\E[\norm{g_k-\nabla f(x_k)}_2^2]\leq \sigma^2(1+\norm{\nabla f(x_k)}_2^2)$. \\ 
\textbf{c)}  Assume the preconditioner $H_k$ has bounded eigenvalues, i.e., $\sup_k\lambda_{\max}(H_k)\leq\lambda_{H,\max}<\infty$ and $\inf_k\lambda_{\min}(H_k)\geq\lambda_{H,\min}>0.$\\
\end{assumption}
\vspace{-1.5em}
For Assumptions~\ref{assumption:func_noise}a) and b), these conditions are standard for stochastic first-order methods (in fact, Assumption~\ref{assumption:func_noise}b) is even milder than the commonly assumed condition \(\E[\|g_k-\nabla f(x_k)\|_2^2]\leq \sigma^2\)). Assumption~\ref{assumption:func_noise}c) requires the preconditioner \(H_k\) to be upper bounded and positive definite, which is guaranteed by the implementation of the Schur--Newton method, the regularization step in \cref{eq:cd_q_state}, and Proposition~\ref{prop:precond_bound}. Specifically, (i) the upper bound follows from \cref{eq:cd_q_state} where an \(\epsilon I\) regularization is added to ensure a lower bound on \(C^L_k\); then, applying the \(-\tfrac{1}{4}\) exponent yields an upper bound while the operator \(\mathcal{D}\!\mathcal{Q}\) keeps the quantization error bounded. (ii) The strict positive definiteness (i.e., the lower bound) is ensured by the Gershgorin Circle Theorem and the Schur--Newton method. In particular, if \(S_k\) denotes the inner matrix in \cref{eq:cd_root} such that \(\mathcal{D}(\hat{L}_k)=\mathcal{D}\!\mathcal{Q}\bigl(S_k^{-1/4}\bigr)\), then the Schur--Newton method (applied for a limited number of steps) yields a diagonally dominant matrix \(Z_k\) approximating \(S_k^{-1/4}\); writing \(\mathcal{D}(\hat{L}_k)=Z_k+E^Z_k\) with \(E^Z_k\) denoting the quantization error, the Gershgorin Circle Theorem guarantees that if \(Z_k\) is strictly diagonally dominant and \(\|E^Z_k\|\) is sufficiently small, then \(Z_k+E^Z_k\) remains strictly positive definite, as further supported by Proposition~\ref{prop:precond_bound}. Empirical evidence in \cref{fig:eign} also demonstrates that the eigenvalues of the dequantized preconditioners \(\calD(\hat{L}_k)\) and \(\calD(\hat{R}_k)\) remain positive throughout training, further validating Assumption~\ref{assumption:func_noise}c).

\begin{figure*}[ht]
	\begin{center}
		\setlength{\tabcolsep}{0.0pt}  
		\scalebox{1.04}{\begin{tabular}{cccc} 
            \includegraphics[width=0.24\linewidth]{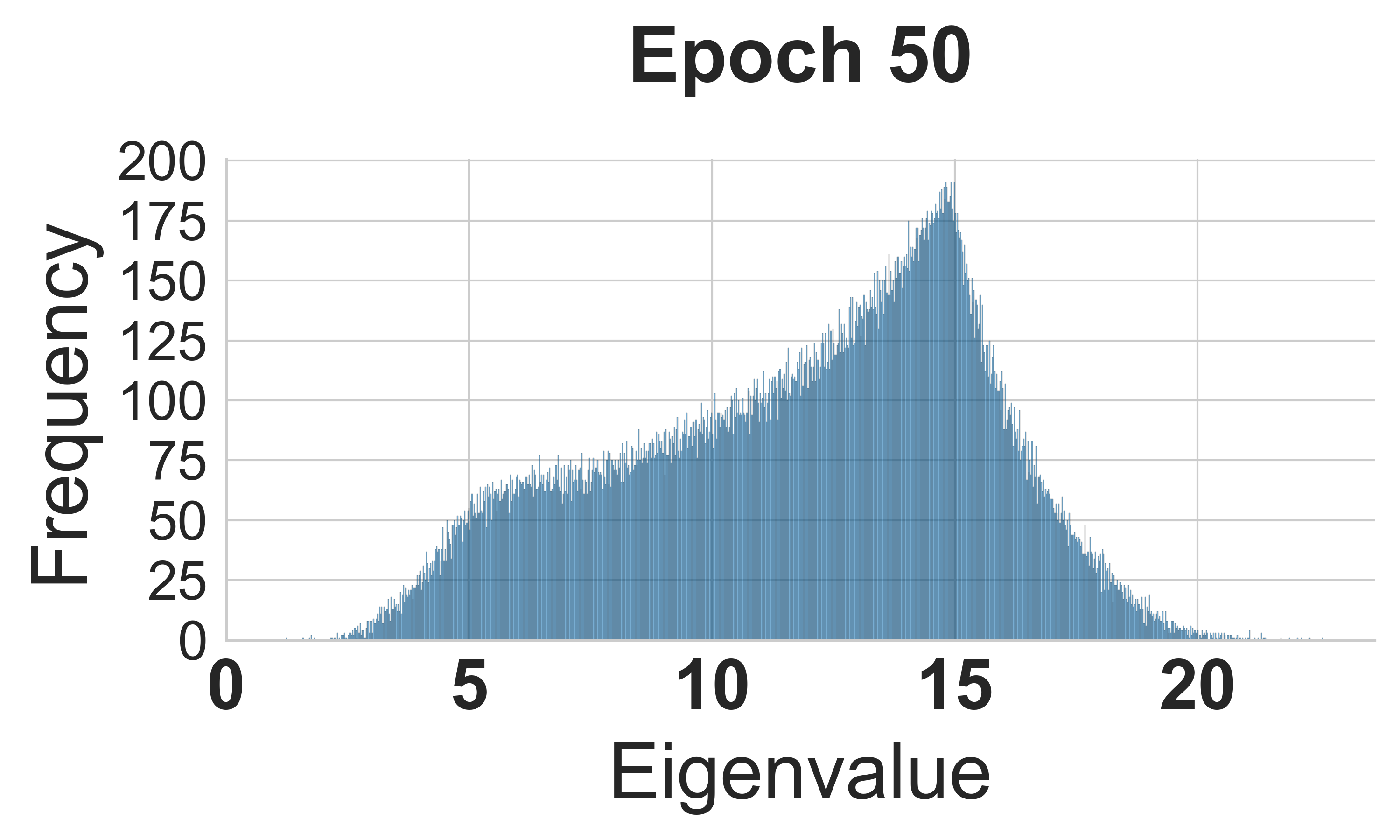}&
            \includegraphics[width=0.24\linewidth]{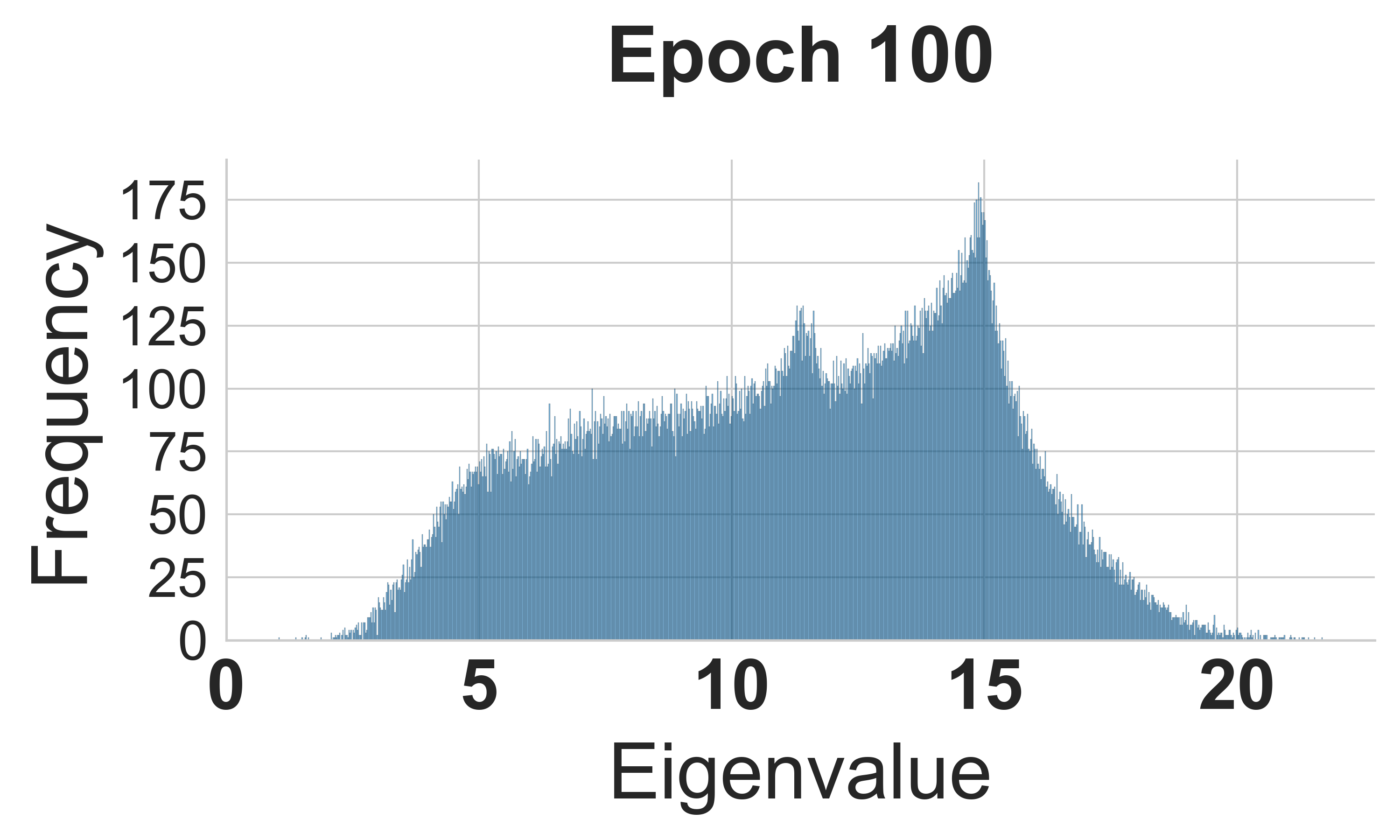}&
            \includegraphics[width=0.24\linewidth]{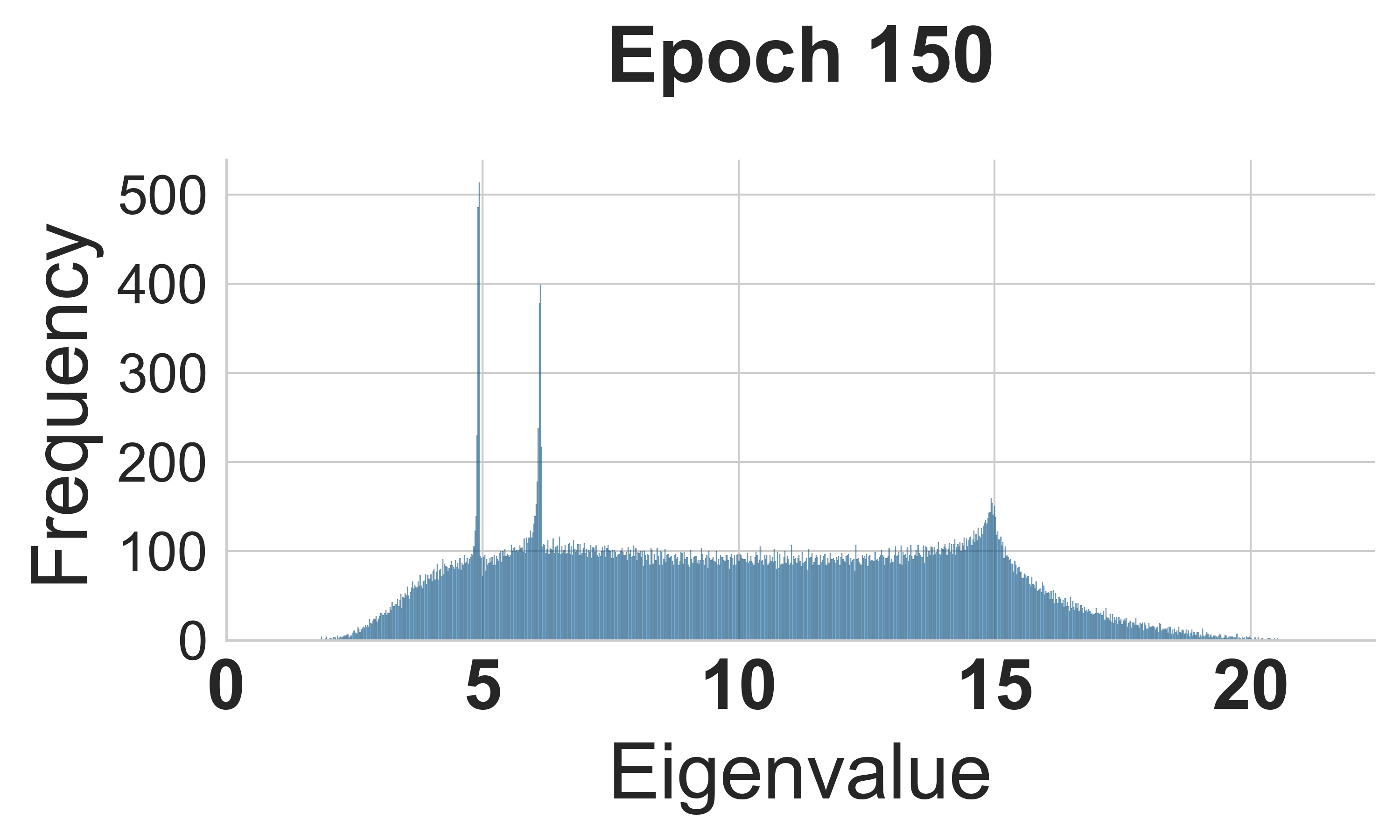}& 
            \includegraphics[width=0.24\linewidth]{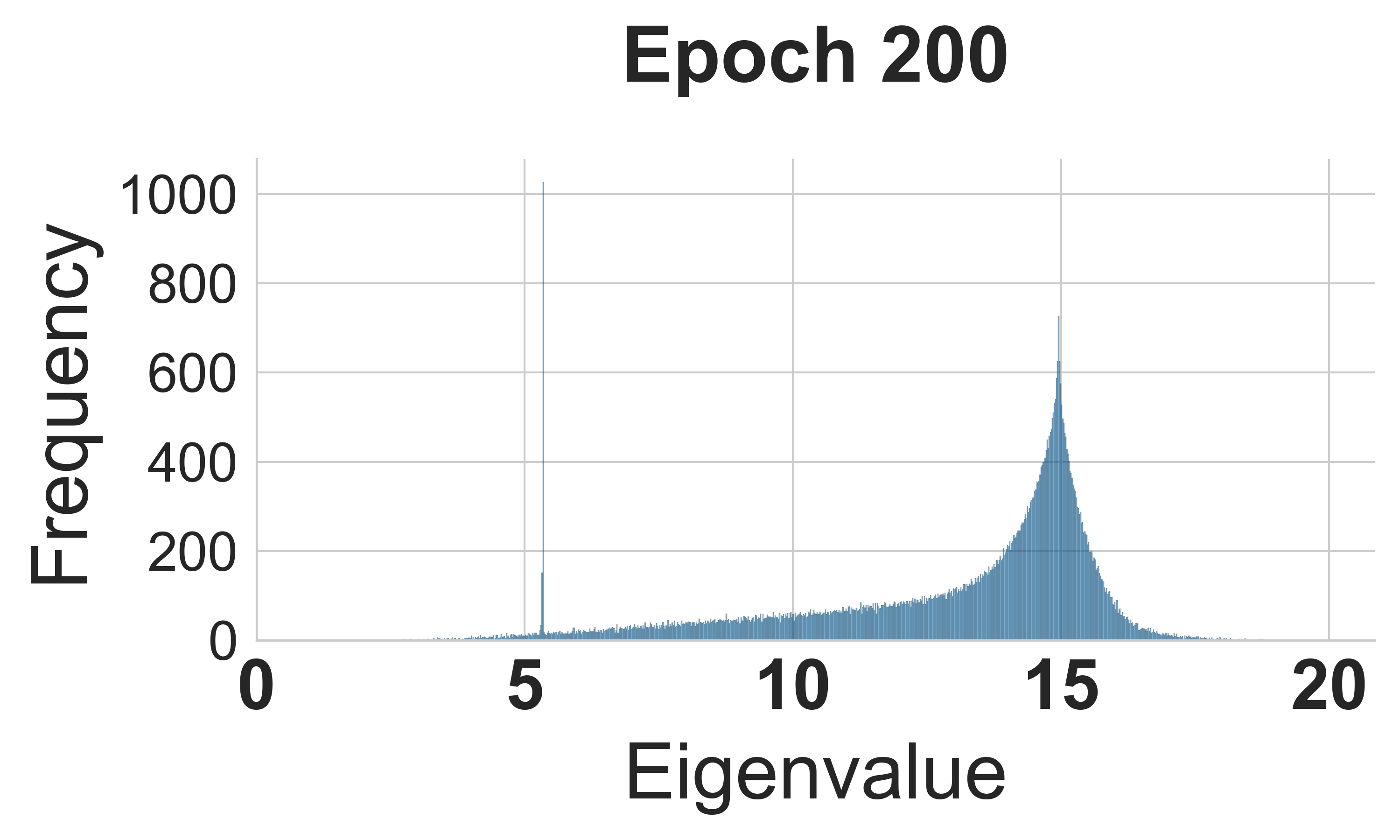}
		\end{tabular}}
	\end{center}
	\vspace{-2.3em}
	\caption{Eigenvalue frequency of the dequantized preconditioners $\calD(\hat{L})$ and $\calD(\hat{R})$ of VGG-19 on CIFAR-100 at 50, 100, 150, and 200 training epochs, all eigenvalues are greater than 0.}
\label{fig:eign}
\vspace{-1em}
\end{figure*}

\vspace{-0.5em}
\begin{proposition}
\label{prop:precond_bound}

For the 4-bit Shampoo in \cref{alg:4bit}, let $M_k:= (\calD(\Bar{C}_k^L) \calD(\Bar{C}_k^L)^T + \lambda_{\max}^L \epsilon I_m)^{-1/4}$, if $\norm{M_k}_{{\rm off},\max}\leq C_B$, then its preconditioners hold that
\[
\setlength{\abovedisplayskip}{5pt}
\setlength{\belowdisplayskip}{5pt}
\setlength{\abovedisplayshortskip}{5pt}
\setlength{\belowdisplayshortskip}{5pt}
\calD(\hat L_k)\preceq  M_k+C_B n_k 2^{-b} I,
\]
where $\norm{\cdot}_{{\rm off},\max}$ is the maximal absolute value of all off-diagonal entries and $n_k$ is the number of rows in $W_k$.  Furthermore, if for every row index \(i\) it holds that $|[M_k]_{ii}|>\qty(1+\frac{2}{2^b-1})\sum_{j\neq i}|[M_k]_{ij}|$, then $\calD(\hat L_k)\succ0$.
\end{proposition}
\vspace{-0.5em}
This proposition shows that the sequence  $\{\calD(\hat L_k)\}$ can be bounded above and below. Now we are ready to derive the convergence, and state the main results below. 

\vspace{-0.5em}
\begin{theorem}
\label{thm:general_psgd}
Suppose Assumption \ref{assumption:func_noise} holds. Let $\eta_k=\frac{c}{\sqrt{T+1}}$ with $c\in\left(0,\frac{\lambda_{H,\min}}{2L(1+\sigma^2)\lambda_{H,\max}^2}\right)$, then we have 
\[
\setlength{\abovedisplayskip}{3pt}
\setlength{\belowdisplayskip}{3pt}
\setlength{\abovedisplayshortskip}{3pt}
\setlength{\belowdisplayshortskip}{3pt}
\E\left[\norm{\nabla f(\bar x_T)}_2^2\right]\leq\frac{2(f(x_0)-\bar f+c^2L\sigma^2\lambda_{H,\max}^2)}{c\lambda_{H,\min}\sqrt{T+1}},
\]
where $\bar x_T$ is randomly selected from $\{x_0,x_1,...,x_T\}$ and $\bar f:=\min_{x\in\R^d}\;f(x)$.
\end{theorem}
\vspace{-0.5em}
See its proof in \cref{app-proof}. Theorem~\ref{thm:general_psgd} shows that for smooth nonconvex training loss, our 4-bit Shampoo with SGD as the base optimizer can converge at the rate  of $\mathcal{O}(\frac{1}{\sqrt{T}})$. This convergence rate is optimal as shown in \cite{carmon2021lower}, indicating the high efficiency of our proposed \cref{alg:4bit}.

\subsection{Nonsmooth Nonconvex Training Loss}
In this subsection, we analyze the nonsmooth nonconvex training loss function, particularly in cases where the activation function is nonsmooth, such as the ReLU in ResNet \cite{he2016deep}. The iterative scheme can be written as: 
\[
\setlength{\abovedisplayskip}{3pt}
\setlength{\belowdisplayskip}{3pt}
\setlength{\abovedisplayshortskip}{3pt}
\setlength{\belowdisplayshortskip}{3pt}
x_{k+1}=x_k-\eta_kH_k(d_k+\xi_k),
\]
where $d_k\in\partial f(x_k)$, $\partial f$ denotes the subgradient of $f$, and $\{\xi_k\}$ is the sequence of the random noise in the subgradient.  Relevant concepts are provided in \cref{app-nonsmooth}. Given a process $\{\xi_i\}_{i=0}^\infty$, let \( \mathcal{F}_k \) denote the history up to iteration \( k \). To this end, we introduce the necessary assumptions.

\vspace{-0.5em}
\begin{assumption}
\label{assumption:nonsmooth}
\textbf{a)} The function \( f \) is \( \ell \)-Lipschitz continuous. Additionally, \( f \) is a Whitney stratifiable function.\\
\textbf{b)} The noise in the subgradient is unbiased and has bounded variance 
\[
\setlength{\abovedisplayskip}{3pt}
\setlength{\belowdisplayskip}{3pt}
\setlength{\abovedisplayshortskip}{3pt}
\setlength{\belowdisplayshortskip}{3pt}
\mathbb{E}[\xi_k | \mathcal{F}_{k-1}] = 0, \quad \mathbb{E}[\norm{\xi_k}_2^2 | \mathcal{F}_{k-1}] \leq \sigma^2, 
\]
\textbf{c)} For any convergent subsequence \( x_{k_j} \to \bar{x} \), we have \( \lim_{N \to \infty} \frac{1}{N} \sum_{j=1}^N H_{k_j} = \bar{H} \) for some positive definite matrix \( \bar{H} \). Additionally, \( \sup_{k \geq 0} \lambda_{\max}(H_k) \leq M \).
\end{assumption}
\vspace{-0.3em} 

The class of Lipschitz continuous functions is broad and includes pathological cases where subgradient flows fail to converge to stationary points \cite{daniilidis2020pathological}. To address this, we focus on Whitney stratifiable functions, which generalize most practical cases, including loss functions in neural networks with nonsmooth activations like ReLU \cite{bolte2021conservative,davis2020stochastic}. Assumption \ref{assumption:nonsmooth}c) requires only Cesàro summability of \( \{H_k\} \), a mild condition crucial for handling non-time-homogeneity.

\vspace{-0.3em}
\begin{theorem}
Suppose that Assumption \ref{assumption:nonsmooth} holds and  the sequence \( \{x_k\} \) remains within a compact set. If the learning rate satisfies \( \sum_{k=1}^\infty \eta_k = \infty \) and \( \sum_{k=1}^\infty \eta_k^2 < \infty \), then
\[
\setlength{\abovedisplayskip}{3pt}
\setlength{\belowdisplayskip}{3pt}
\setlength{\abovedisplayshortskip}{3pt}
\setlength{\belowdisplayshortskip}{3pt}
\lim_{k \rightarrow \infty} {\rm dist}(x_k, \Omega) = 0,
\]
where \( \Omega := \{x : 0 \in \partial f(x)\} \) is the set of stationary points.
\end{theorem}
\vspace{-0.5em} 
For a stratifiable function, the result of convergence to the stationary point set is tight. There are no complexity results due to the challenges posed by its complex nonconvexity and nonsmoothness \cite{davis2020stochastic,bolte2021conservative}. This result ensures the convergence of our proposed algorithm on nonsmooth training losses, including those arising in deep neural networks such as ReLU-based architectures.

\vspace{-0.3em}
\section{Experiments} \label{sec-experiment}
\vspace{-0.3em}

In this section, we evaluate our 4-bit Shampoo \cref{alg:4bit} on classical image classification and large language model (LLM) pre-training. We compare its performance against vanilla 4-bit and 32-bit Shampoo when using SGD with momentum (SGDM) \cite{sutskever2013importance} or AdamW \cite{loshchilov2019decoupled} as base optimizer, and the base optimizer itself. For all experiments, we report test accuracy and peak memory usage to assess both algorithmic performance and GPU memory overhead.

\noindent \textbf{Training Setting.} 
Following standard benchmarks for image classification \cite{he2016deep,wightman2021resnet,lee2021vision}, we train VGG-19 \cite{simonyan2014very}, ResNet-34 \cite{he2016deep}, Swin Transformer Tiny (Swin-Tiny) \cite{liu2021swin}, and Vision Transformer Small (ViT-Small) \cite{dosovitskiy2020image} on CIFAR-100 \cite{krizhevsky2009learning} and Tiny-ImageNet \cite{le2015tiny}, as well as ResNet-50 and ViT-Base on ImageNet \cite{deng2009imagenet}. For LLM pre-training, we follow \cite{lialin2023relora,zhao2024galore} to train LLaMA \cite{touvron2023llama} on the C4 dataset \cite{raffel2020exploring} with varying model sizes. Training hyperparameters for Shampoo match those of the base optimizer, except that the base optimizer is trained for additional epochs in image classification to achieve comparable performance. All experiments are conducted on a single NVIDIA A100 GPU (80GB). Further details are provided in \cref{app-exp_hyp}.

\begin{table}[ht]
\caption{Test accuracy (\%) and peak memory (MB) of vanilla 4-bit Shampoo with off-diagonal and original block-wise quantization for VGG-19 on CIFAR-100 and Swin-Tiny on Tiny-ImageNet.}
\label{tab:diagonal}
\vspace{-1.6em}
\begin{center}
    \begin{small}
    \renewcommand{\arraystretch}{0.85}
    \setlength{\aboverulesep}{1pt}
    \setlength{\tabcolsep}{2.5pt}  
    \begin{tabular}{l|cc|cc}
        \toprule
        Model & \multicolumn{2}{c|}{VGG-19} & \multicolumn{2}{c}{Swin-Tiny} \\
             & Accuracy & Memory & Accuracy & Memory \\
        \midrule
         Original & 74.20 & 661.7 & 60.83 & 1126.2\\
        Off-Diagonal & 74.36 & 662.2 & 61.28 & 1126.9 \\
        \bottomrule
    \end{tabular} 
    \end{small}
\end{center}
\vspace{-1.2em}
\end{table}

\begin{table}[ht]
\centering
\caption{Test accuracy (\%) and peak memory (MB) on CIFAR-100. Here VQ denotes vanilla quantization, CQ denotes Cholesky quantization, and EF denotes error feedback.}
\vspace{-0.8em}
\label{tab:cifar100}
\resizebox{\columnwidth}{!}{
\begin{small}
\renewcommand{\arraystretch}{0.85}  
\setlength{\aboverulesep}{1pt}
\setlength{\tabcolsep}{3pt}  
\begin{tabular}{llcc}
\toprule
\textbf{Model} & \textbf{Optimizer} & \textbf{Accuracy} & \textbf{Memory} \\ 
\midrule
\multirow{5}{*}{VGG-19} 
    & SGDM & 74.43 & 597.3 \\ 
    & SGDM + 32-bit Shampoo & 75.02 & 1065.2 \\ 
    & SGDM + 4-bit Shampoo (VQ) & 74.36 & 662.2 \\ 
    & SGDM + 4-bit Shampoo (CQ) & 74.99 & 646.0 \\ 
    & SGDM + 4-bit Shampoo (CQ+EF) & 75.21 & 662.2 \\ 
\midrule
\multirow{5}{*}{ResNet-34} 
    & SGDM & 79.12 & 1254.7 \\ 
    & SGDM + 32-bit Shampoo & 80.69 & 1882.6 \\ 
    & SGDM + 4-bit Shampoo (VQ) & 79.45 & 1341.0 \\ 
    & SGDM + 4-bit Shampoo (CQ) & 80.27 & 1319.5 \\
    & SGDM + 4-bit Shampoo (CQ+EF) & 80.52 & 1341.0 \\
\midrule
\multirow{5}{*}{Swin-Tiny} 
    & AdamW & 78.28 & 1095.3 \\ 
    & AdamW + 32-bit Shampoo & 79.84 & 1248.6 \\ 
    & AdamW + 4-bit Shampoo (VQ) & 78.33 & 1116.8 \\ 
    & AdamW + 4-bit Shampoo (CQ) & 79.29 & 1111.5 \\
    & AdamW + 4-bit Shampoo (CQ+EF) & 79.45 & 1116.8 \\ 
\midrule
\multirow{5}{*}{ViT-Small} 
    & AdamW & 73.00 & 2930.0 \\ 
    & AdamW + 32-bit Shampoo & 77.95 & 3448.9 \\ 
    & AdamW + 4-bit Shampoo (VQ) & 74.56 & 3001.7 \\ 
    & AdamW + 4-bit Shampoo (CQ) & 77.34 & 2983.7 \\
    & AdamW + 4-bit Shampoo (CQ+EF) & 77.74 & 3001.7 \\
\bottomrule
\end{tabular}
\end{small}  }
\vspace{-1em}
\end{table}

\begin{table}[ht]
\centering
\caption{Test accuracy (\%) and peak memory (MB) on Tiny-ImageNet. Here VQ denotes vanilla quantization, CQ denotes Cholesky quantization, and EF denotes error feedback.}
\vspace{-0.8em}
\label{tab:timnet}
\resizebox{\columnwidth}{!}{
\begin{small}
\renewcommand{\arraystretch}{0.85}  
\setlength{\aboverulesep}{1pt}
\setlength{\tabcolsep}{3pt}  
\begin{tabular}{llcc}
\toprule
\textbf{Model} & \textbf{Optimizer} & \textbf{Accuracy} & \textbf{Memory} \\ 
\midrule
\multirow{4}{*}{VGG-19} 
    & SGDM & 62.19 & 1632.8 \\
    & SGDM + 32-bit Shampoo & 63.36 & 2102.5 \\
    & SGDM + 4-bit Shampoo (VQ) & 62.34 & 1697.8 \\
    & SGDM + 4-bit Shampoo (CQ+EF) & 63.51 & 1697.8 \\
\midrule
\multirow{4}{*}{ResNet-34} 
    & SGDM & 68.27 & 4221.3 \\
    & SGDM + 32-bit Shampoo & 69.11 & 4846.0 \\
    & SGDM + 4-bit Shampoo (VQ) & 68.43 & 4307.7 \\
    & SGDM + 4-bit Shampoo (CQ+EF) & 68.88 & 4307.7 \\
\midrule
\multirow{4}{*}{Swin-Tiny} 
    & AdamW & 60.74 & 1105.5 \\
    & AdamW + 32-bit Shampoo & 62.73 & 1256.8 \\
    & AdamW + 4-bit Shampoo (VQ) & 61.28 & 1126.9 \\
    & AdamW + 4-bit Shampoo (CQ+EF) & 62.81 & 1126.9 \\
\midrule
\multirow{4}{*}{ViT-Small} 
    & AdamW & 55.21 & 2944.2 \\
    & AdamW + 32-bit Shampoo & 58.11 & 3468.1 \\
    & AdamW + 4-bit Shampoo (VQ) & 56.47 & 3016.0 \\
    & AdamW + 4-bit Shampoo (CQ+EF) & 57.51 & 3016.0 \\
\bottomrule
\end{tabular}
\end{small}  }
\vspace{-1em}
\end{table}

\begin{figure*}[ht]
\begin{center}
    \includegraphics[width=0.66\linewidth]{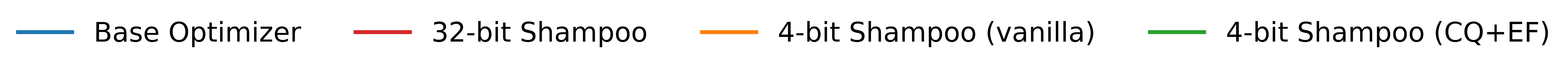} 
    \vspace{-0.4em}
    \setlength{\tabcolsep}{10pt}  
    \scalebox{1}{\begin{tabular}{cc} 
        \includegraphics[width=0.44\linewidth]{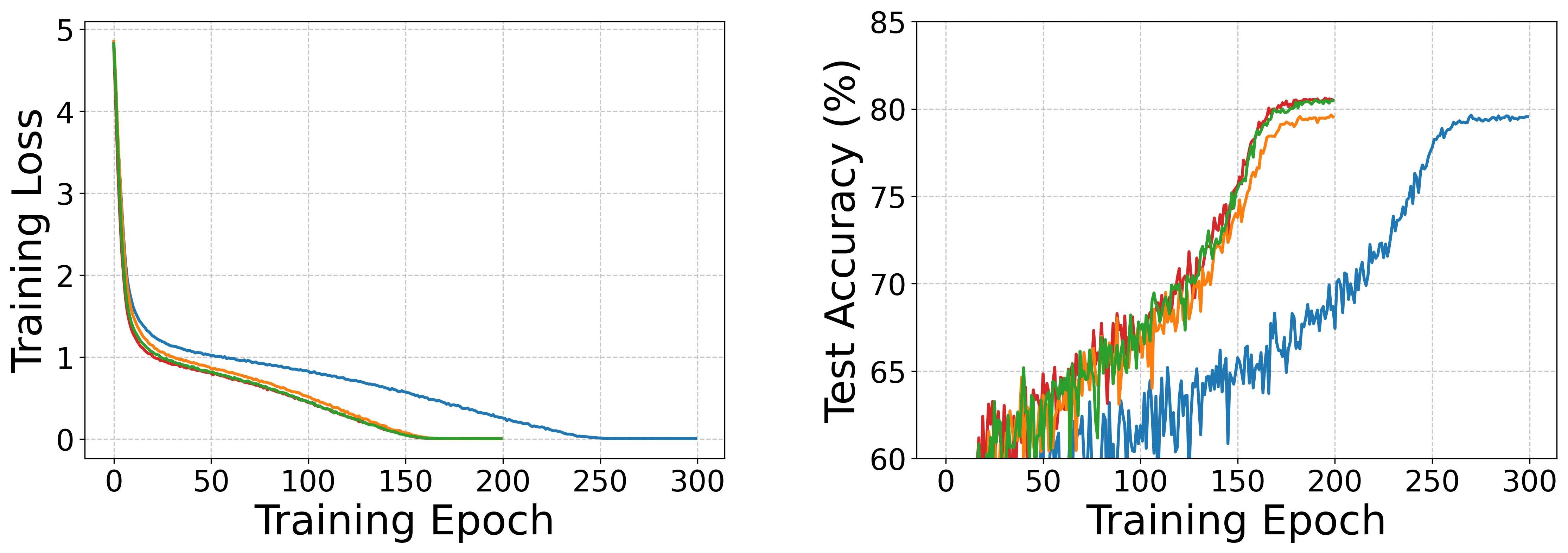} &
        \includegraphics[width=0.44\linewidth]{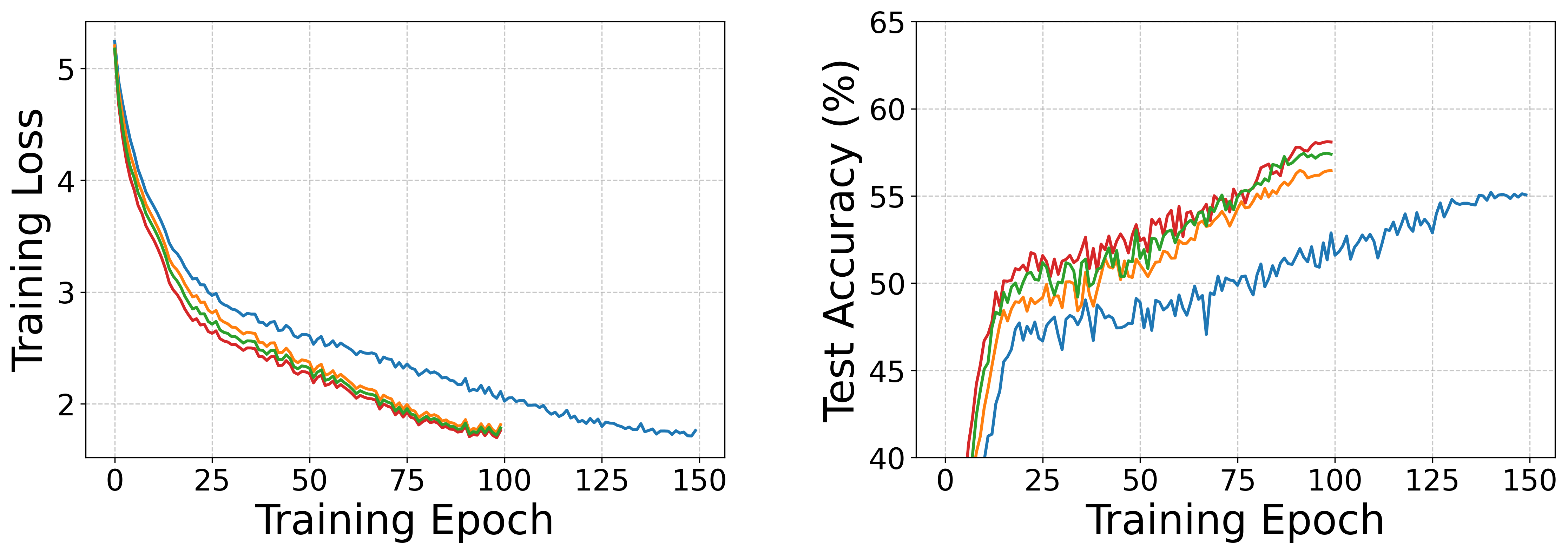} \\
        \footnotesize{(a) SGDM as base optimizer.}& \footnotesize{(b) AdamW as base optimizer.}
    \end{tabular}}
\end{center}
\vspace{-1.5em}
\caption{Comparison of training loss and test accuracy (\%) for training ResNet-34 on CIFAR-100 and ViT-Small on Tiny-ImageNet. The left figure shows ResNet-34 results, and the right figure shows ViT-Small results.}
\label{fig:train_test}
\vspace{-0.8em}
\end{figure*}

\vspace{-0.1em}
\subsection{Test Performance}
\vspace{-0.3em}
To ensure a fair comparison between vanilla 4-bit Shampoo and our method, we apply off-diagonal 4-bit block-wise quantization to Shampoo's preconditioners while retaining diagonal elements in 32-bit, defining this as vanilla 4-bit Shampoo. As shown in \cref{tab:diagonal}, off-diagonal quantization only slightly increases peak memory but improves test performance. Thus, we adopt off-diagonal quantization for vanilla 4-bit Shampoo in all experiments.

As shown in \cref{tab:cifar100}, 4-bit Shampoo with Cholesky quantization consistently outperforms vanilla 4-bit Shampoo. For instance, with SGDM as the base optimizer for ResNet-34 on CIFAR-100, it achieves 80.27\% test accuracy versus 79.45\% for vanilla 4-bit Shampoo. Similarly, with AdamW for ViT-Small on CIFAR-100, it reaches 77.34\% compared to 74.56\%. This improvement stems from Cholesky quantization’s ability to recover preconditioners from Cholesky factors, better preserving the spectral properties of 32-bit Shampoo preconditioners (\cref{sec-cholesky}).

Moreover, experimental results in \cref{tab:cifar100} validate the effectiveness of the error compensation strategy for Cholesky factors introduced in \cref{CholeskyCompensation}. With SGDM as the base optimizer for ResNet-34 on CIFAR-100, 4-bit Shampoo with compensated Cholesky decomposition improves test accuracy by 0.25\% over 4-bit Cholesky quantization. Similarly, with AdamW for ViT-Small on Tiny-ImageNet, it achieves a 0.4\% improvement. This consistent gain stems from EF, which retains and integrates quantization errors from previous steps into the updated Cholesky factors before each quantization, iteratively minimizing quantization errors.

Experimental results on larger image classification datasets (\cref{tab:timnet,tab:imagenet}) further validate the superiority of our 4-bit Shampoo with compensated Cholesky quantization. On Tiny-ImageNet, it consistently improves test accuracy by over 0.45\% compared to vanilla 4-bit Shampoo, whether using SGDM or AdamW as the base optimizer. For ResNet-50 and ViT-Base on ImageNet, it increases test accuracy by 0.27\% and 0.73\%, respectively, achieving performance close to the original 32-bit Shampoo.  

\begin{table}[ht]
\centering
\caption{Comparison of test accuracy (\%), wall-clock time (min), and peak memory (MB) on the ImageNet dataset.
}
\vspace{-0.8em}
\label{tab:imagenet}
\resizebox{\columnwidth}{!}{
\begin{small}
\renewcommand{\arraystretch}{0.85}
\setlength{\aboverulesep}{1pt}
\begin{tabular}{llccc}
\toprule
\textbf{Model} & \textbf{Optimizer} & \textbf{Accuracy} & \textbf{Time} & \textbf{Memory} \\ 
\midrule
\multirow{4}{*}{ResNet-50} 
    & Base & 77.56 & 2106.4 & 11356.2 \\
    & 32-bit & 78.06 & 1841.1 & 11986.4 \\
    & 4-bit (VQ) & 77.73 & 1882.8 & 11445.2 \\
    & 4-bit (ours) & 78.00  & 1899.4 & 11445.2 \\
\midrule

\multirow{4}{*}{ViT-Base} 
    & Base & 72.59 & 1741.6 & 11839.7  \\
    & 32-bit & 75.47 &  1392.4 & 13319.1 \\
    & 4-bit (VQ) & 72.28 & 1406.2 & 12052.3 \\
    & 4-bit (ours) & 75.01 & 1409.6 & 12052.3 \\
\bottomrule
\end{tabular}
\end{small} }
\vspace{-0.8em}
\end{table}

\begin{table}[ht]
\centering
\caption{Comparison of test perplexity (PPL, lower is better), update time (min), and peak memory (GB) on the C4 dataset.}
\vspace{-0.8em}
\label{tab:llm}
\resizebox{\columnwidth}{!}{
\begin{small}
\renewcommand{\arraystretch}{0.85}
\setlength{\aboverulesep}{1pt}
\begin{tabular}{llccc}
\toprule
\textbf{Model} & \textbf{Optimizer} & \textbf{PPL} & \textbf{Time} & \textbf{Memory} \\ 
\midrule
\multirow{4}{*}{LLaMA-130M} 
    & Base & 27.32 & 162.9 & 45.9  \\
    & 32-bit & 26.93 & 169.1 & 48.2 \\
    & 4-bit (VQ) & 28.08 & 170.5 & 46.2 \\
    & 4-bit (ours) & 26.98 & 178.9 & 46.2 \\
\midrule

\multirow{4}{*}{LLaMA-350M} 
    & Base & 24.29 & 431.7 & 52.9  \\
    & 32-bit & 24.07 & 443.8 & 58.8 \\
    & 4-bit (VQ) & 25.14 & 445.3 & 53.7 \\
    & 4-bit (ours) & 23.99 & 456.2 & 53.7 \\
\midrule

\multirow{4}{*}{LLaMA-1B} 
    & Base & 48.39 & 403.7 & 59.0 \\
    & 32-bit & \multicolumn{3}{c}{\textit{Out of GPU Memory}} \\
    & 4-bit (VQ) & 48.53 & 411.4 & 61.9 \\
    & 4-bit (ours) & 46.31 & 425.0 & 61.9 \\
\bottomrule
\end{tabular}
\end{small} }
\vspace{-0.8em}
\end{table}

For LLM pre-training experiments (\cref{tab:llm}), our 4-bit Shampoo with compensated Cholesky quantization consistently achieves lower test perplexity than vanilla 4-bit Shampoo and the base optimizer. Its performance closely matches 32-bit Shampoo, provided the 32-bit version fits within GPU memory. These results demonstrate the effectiveness of our quantization strategy in preserving test performance for large-scale neural network training.

\vspace{-0.1em}
\subsection{Memory and Computational Efficiency}
\vspace{-0.3em}
For GPU memory usage, \cref{tab:cifar100,tab:timnet,tab:llm,tab:imagenet} show that 4-bit quantization significantly reduces the peak memory consumption of 32-bit Shampoo. For instance, with SGDM as the base optimizer for ResNet-34 on CIFAR-100, 4-bit Shampoo lowers peak memory by over 540MB, reducing usage by more than 28\%. Similarly, with AdamW for LLaMA-350M on C4, it reduces peak memory by 5.1GB. Moreover, when training LLaMA-1B, 32-bit Shampoo exceeds GPU memory limits on an A100 (80GB), whereas our 4-bit Shampoo runs efficiently with strong test performance.

Additionally, as shown in \cref{tab:cifar100}, 4-bit Cholesky quantization further reduces peak GPU memory usage compared to vanilla quantization. For example, when training ResNet-34 on CIFAR-100 with SGDM, it reduces peak memory by 21.5MB, accounting for 25\% of the 86.3MB overhead introduced by vanilla 4-bit Shampoo’s preconditioners. This efficiency arises from storing only the lower triangular Cholesky factors \(\Bar{C}_k^L, \Bar{C}_k^R\), which require half the memory of full matrices \(L_k, R_k\) (\cref{sec-cholesky}). Thus, 4-bit Shampoo with Cholesky quantization achieves additional memory savings over vanilla 4-bit Shampoo. See \cref{app-mem} for further details.

For computational efficiency, \cref{tab:imagenet,tab:llm} show that the overhead introduced by compensated Cholesky quantization over vanilla 4-bit quantization is minimal. When training ResNet-50 and ViT-Base on ImageNet, the additional computation time is under 20 minutes, accounting for less than 1\% of the total training time. For LLaMA training on the C4 dataset, it adds less than 15 minutes, contributing to under 5\% of the total training time.

\vspace{-0.1em}
\subsection{Ablation Study}
\vspace{-0.3em}
\noindent \textbf{Effects of \(\beta\) and \(\beta_e\).} Following modern Shampoo algorithms \cite{anil2020scalable,shi2023distributed}, we maintain an exponential moving average of Cholesky factors and error states. \cref{tab:beta} shows the robustness of our method to momentum coefficients \(\beta,\beta_e\).

\begin{table}[h]
\vspace{-0.8em}
    \centering
    \setlength{\tabcolsep}{5pt}  
    \renewcommand{\arraystretch}{0.2}  
    \small  
    \caption{Test accuracy (\%) for ResNet-34 on CIFAR-100.}
    \label{tab:beta}
    \vspace{-0.9em}
    \begin{tabular}{l|ccccccc}
        \toprule
        \(\beta,\beta_e\) & 0.6 & 0.7 & 0.8 & 0.9 & 0.95 & 0.98\\
        \midrule
        Accuracy & 80.40 & 80.36 & 80.44 & 80.47 &  80.52 & 80.30 \\
        \bottomrule
    \end{tabular}
\vspace{-1em}
\end{table}  

\noindent \textbf{More Optimizers.} We further evaluate RMSprop as the base optimizer. As shown in \cref{tab:more}, our 4-bit Shampoo consistently outperforms vanilla 4-bit Shampoo while reducing memory usage compared to the 32-bit version.

\begin{table}[h]
\vspace{-0.8em}
    \centering
    \setlength{\tabcolsep}{4pt}  
    \renewcommand{\arraystretch}{0.5}  
    \small  
    \caption{Test accuracy (\%) and peak memory (MB) for Swin-Tiny on CIFAR-100 with RMSprop as the base optimizer.}
    \vspace{-0.9em}
    \label{tab:more}
    \begin{tabular}{lcc}
        \toprule
        \textbf{Optimizer} & \textbf{Accuracy} & \textbf{Memory} \\
        \midrule
        RMSprop & 74.35 & 1066.1 \\
        RMSprop+32-bit Shampoo & 75.67 & 1219.5 \\
        RMSprop+4-bit Shampoo (VQ) & 74.82 & 1087.5 \\
        RMSprop+4-bit Shampoo (ours) & 75.31 & 1087.5 \\
        \bottomrule
    \end{tabular}
\vspace{-1.5em}
\end{table}

\vspace{-0.3em}
\section{Conclusion}
\vspace{-0.5em}
We introduce 4-bit Shampoo, a memory-efficient preconditioned gradient method that significantly reduces GPU memory usage while maintaining performance comparable to 32-bit Shampoo. By applying Cholesky quantization, we store only 4-bit lower triangular Cholesky factors, halving memory costs while better preserving spectral properties of preconditioners. An error feedback mechanism further mitigates quantization loss by compensating for errors at each step. We prove convergence in nonconvex settings, and our method achieves strong performance on image classification benchmarks and LLM pre-training.  

\noindent \textbf{Limitations.} (a) Our Cholesky quantization and error feedback strategy were only tested with Shampoo, though they are generalizable to other preconditioned gradient methods, which we leave for future work. (b) Due to limited GPU resources, we evaluated 4-bit Shampoo only on image classification and LLM pre-training, leaving tasks such as object detection, video generation, and large-scale model fine-tuning for future exploration.

{
    \small
    \bibliographystyle{ieeenat_fullname}
    \bibliography{main}
}

\clearpage
\setcounter{page}{1}
\maketitlesupplementary
\appendix

\section{Practical 32-bit Shampoo} \label{app-32bit}
In this section, we provide the practical 32-bit Shampoo introduced in \cref{sec-shampoo} and summarize it in \cref{alg:32bit}.

\begin{algorithm} 
\caption{Practical 32-bit Shampoo} \label{alg:32bit}
\textbf{Input:} initial weight $W_0 \in \mathbb{R}^{m \times n}$, initial preconditioning matrices $L_0 = \epsilon I_m$, $R_0 = \epsilon I_n$, $\hat{L}_0 = I_m$, $\hat{R}_0 = I_n$. Total update steps $T$, interval of updating preconditioners $T_1$ and $T_2$, momentum parameter $\beta \in (0,1)$. First-order optimizer $\mathcal{F}$ with initial optimizer state $s_0$.\\
\textbf{Output:} final weight $W_T$.
\begin{algorithmic}[1]
\For{$k = 1, 2, \ldots, T$}
    \State Compute gradient $G_k = \nabla \mathcal{L}_k(W_k)$
    \If{$k \% T_1 \equiv 0$}
        \State $\Bar{L}_k = \beta L_{k-1} + (1 - \beta) G_k G_k^T$
        \State $\Bar{R}_k = \beta R_{k-1} + (1 - \beta) G_k^T G_k$
    \Else
        \State $L_k = L_{k-1}, \ R_k = R_{k-1}$
    \EndIf
    \If{$k \% T_2 \equiv 0$}
        \State Compute maximum singular values $\lambda_{\max}^L$ and $\lambda_{\max}^R$ of $L_k$ and $R_k$ by power iteration
        \State Compute $\hat{L}_k = (L_k + \lambda_{\max}^L \epsilon I_m)^{-1/4}$ and $\hat{R}_k = (R_k + \lambda_{\max}^R \epsilon I_n)^{-1/4}$ by Schur-Newton iteration 
    \Else
        \State $\hat{L}_k = \hat{L}_{k-1}$; \quad $\hat{R}_k = \hat{R}_{k-1}$
    \EndIf
    \State $\hat{G}_k = \hat{L}_k G_k \hat{R}_k$; \quad $\tilde{G}_k = (\|G_k\|_F / \|\hat{G}_k\|_F) \cdot \hat{G}_k $
    \State $W_k, s_k = \mathcal{F}(W_{k-1}, s_{k-1}, \tilde{G}_k)$
\EndFor
\end{algorithmic}
\end{algorithm}

\section{Proofs in Theoretical Analysis} \label{app-proof}
We vectorize the update scheme as follows. Starting with the matrix form:
\[
W_{k+1} = W_k - \eta_k \calD(\hat L_k) G_k \calD(\hat R_k),
\]
and applying vectorization, we get:
\[
\vect(W_{k+1}) = \vect(W_k) - \eta_k \left( \calD(\hat R_k) \otimes \calD(\hat L_k) \right) \vect(G_k).
\]
Let \( x_k := \vect(W_k) \), \( g_k := \vect(G_k) \), and \( H_k :=  \calD(\hat R_k) \otimes \calD(\hat L_k)  \).  we obtain the vectorized update scheme:
\begin{equation}
x_{k+1}=x_k-\eta_kH_kg_k,
\end{equation}
where \( \{H_k\} \) is a sequence of positive definite matrices.

\begin{proposition}
\label{prop:DQ_bound}
For a $b$-bit quantization and any vector \( x \in \mathbb{R}^d \), the following bound holds:
\[
\norm{\calD(\calQ(x)) - x}_\infty \leq \frac{\norm{x}_\infty}{2^b}.
\]
\end{proposition}

\begin{proof}
Consider any real number \( a \in [-1, 1] \). In a $b$-bit quantization system, the interval between two consecutive representable values is given by \( \Delta = \frac{2}{2^b} = \frac{1}{2^{b-1}} \). Thus, the quantization error satisfies \( |\calQ(a) - a| \leq \frac{\Delta}{2} = \frac{1}{2^b} \).

For any vector \( x \in \mathbb{R}^d \), we apply the definitions of the operators \( \calQ \) and \( \calD \) as follows:
\[
\begin{aligned}
&\norm{\calD(\calQ(x)) - x}_\infty\\
=& \norm{\norm{x}_\infty \calQ\left(\frac{x}{\norm{x}_\infty}\right) - \norm{x}_\infty \frac{x}{\norm{x}_\infty}}_\infty \\
=& \norm{x}_\infty \norm{\calQ\left(\frac{x}{\norm{x}_\infty}\right) - \frac{x}{\norm{x}_\infty}}_\infty \\
\leq& \norm{x}_\infty \cdot \frac{1}{2^b}.
\end{aligned}
\]
This completes the proof.
\end{proof}

\begin{proposition}
\label{prop:precond_bound_append}
For the 4-bit Shampoo in \cref{alg:4bit}, let $M_k:= (\calD(\Bar{C}_k^L) \calD(\Bar{C}_k^L)^T + \lambda_{\max}^L \epsilon I_m)^{-1/4}$, if $\norm{M_k}_{{\rm off},\max}\leq C_B$, then its preconditioners hold that
\[
\calD(\hat L_k)\preceq  M_k+C_B n_k 2^{-b} I,
\]
where $\norm{\cdot}_{{\rm off},\max}$ is the maximal absolute value of all off-diagonal entries and $n_k$ is the number of rows in $W_k$.  Furthermore, if for every row index \(i\) it holds that $|[M_k]_{ii}|>\qty(1+\frac{2}{2^b-1})\sum_{j\neq i}|[M_k]_{ij}|$, then $\calD(\hat L_k)\succ0$.
\end{proposition}

\begin{proof}
Unroll the update in Step 4, we have
\[
\begin{aligned}
&L_k\\
=&\beta L_{k-1}+(1-\beta)G_kG_k^T\\
=&\beta(\beta L_{k-2}+(1-\beta)G_{k-1}G_{k-1}^T)+(1-\beta)G_kG_k^T\\
=&\beta^2L_{k-2}+(1-\beta)(G_kG_k^T+\beta G_{k-1}G_{k-1}^T)\\
&...\\
=&\beta^kL_0+(1-\beta)\sum_{i=0}^{k-1}\beta^iG_{k-i}G_{k-i}^T\\
=&\beta^kL_0+(1-\beta)\sum_{i=0}^{k-1}\beta^iG_{k-i}G_{k-i}^T\\
\succeq&0.
\end{aligned}
\]
Thus Step 11 is well-defined. Since only off-diagonal part is quantized, by Step 6, we have 
\begin{equation}
\label{eq:DS_decom}
\begin{aligned}
\calD(\hat L_k)=&\calD(\calQ(M_k))\\
=&\calD(\calQ(S_k-{\rm Diag}(M_k)))+{\rm Diag}(M_k)\\
=&M_k-{\rm Diag}(M_k)+{\rm Diag}(M_k)+E_k\\
=&M_k+E_k,
\end{aligned}
\end{equation}
where $E_k=(M_k-{\rm Diag}(M_k))-\calD(\calQ(M_k-{\rm Diag}(M_k)))$. By Proposition \ref{prop:DQ_bound}, we have
\[
\begin{aligned}
&\norm{E_k}_{\max}\\\leq&\norm{M_k-{\rm Diag}(M_k)}_{\max}2^{-b}\\
\leq&\norm{(\calD(\Bar{C}_k^L) \calD(\Bar{C}_k^L)^T + \lambda_{\max}^L \epsilon I_m)^{-1/4}}_{{\rm off},\max}2^{-b}\\
\leq &C_B2^{-b},
\end{aligned}
\]
where $\norm{\cdot}_{\max}$ is the largest entry in magnitude of a matrix. Note that for any $x\in\R^d$, 
\[
|x^TE_kx|\leq C_B2^{-b}(e^T|x|)^2\leq C_Bn_k2^{-b}\norm{x}^2,
\]
where $e$ is the vector with all elements being 1 and $|\cdot|$ is the operator of taking element-wise absolute value. Therefore, we have
\[
\begin{aligned}
 &\calD(\hat L_k)\\
=&(\calD(\Bar{C}_k^L) \calD(\Bar{C}_k^L)^T + \lambda_{\max}^L \epsilon I_m)^{-1/4}+E_k,  \\
\preceq&(\calD(\Bar{C}_k^L) \calD(\Bar{C}_k^L)^T + \lambda_{\max}^L \epsilon I_m)^{-1/4}+C_Bn_k2^{-b}I.
\end{aligned}
\]
Moreover, if $|[M_k]_{ii}|>\qty(1+\frac{2}{2^b-1})\sum_{j\neq i}|[M_k]_{ij}|$ for any row index $i$, then by \cref{eq:DS_decom}, we have
\[
\begin{aligned}
&\Big|\left[\calD(\hat L_k)\right]_{ii}\Big|-\sum_{j\neq i}\Big|\left[\calD(\hat L_k)\right]_{ij}\Big|\\
\geq&\qty(|\qty[M_k]_{ii}|-|\qty[E_k]_{ii}|)-\qty(\sum_{j\neq i}|\qty[M_k]_{ij}|+\sum_{j\neq i}|\qty[E_k]_{ij}|)\\
\geq&(1-2^{-b})|\qty[M_k]_{ii}|+(1+2^{-b})\sum_{j\neq i}|\qty[M_k]_{ij}|\\
>&0,
\end{aligned}
\]
where the second inequality follows from Proposition \ref{prop:DQ_bound} and the last inequality follows from the strongly diagonal dominance. By Gershgorin Circle Theorem, we have $\calD(\hat L_k)\succ0$. This completes the proof.
\end{proof}
Given a matrix $S$, the proof of Proposition \ref{prop:precond_bound_append} shows that if we quantize only the off-diagonal entries of $S$ while keeping the diagonal entries, the quantization error $E$ satisfies $\norm{E}_\infty\leq 2^{-b} \norm{S}_{{\rm off},\infty}$. However, if the entire $S$ is quantized, the error becomes $2^{-b} \norm{S}_\infty$. When the diagonal entries of $S$ dominate each row, this selective quantization method can significantly reduce the quantization error.

\subsection{Smooth Nonconvex Training Loss} \label{app-smooth}

\begin{theorem}
Suppose Assumption \ref{assumption:func_noise} holds. Let $\eta_k=\frac{c}{\sqrt{T+1}}$ with $c\in\left(0,\frac{\lambda_{H,\min}}{2L(1+\sigma^2)\lambda_{H,\max}^2}\right)$, then we have 
\[
\E\left[\norm{\nabla f(\bar x_T)}_2^2\right]\leq\frac{2(f(x_0)-\bar f+c^2L\sigma^2\lambda_{H,\max}^2)}{c\lambda_{H,\min}\sqrt{T+1}},
\]
where $\bar x_T$ is randomly selected from $\{x_0,x_1,...,x_T\}$, and $\bar f=\min_{x\in\R^d}f(x)$.
\end{theorem}

\begin{proof}
Without any ambiguity, $\norm{\cdot}$ denotes the $L_2$ norm of a vector or the spectral norm of a matrix. By Lipschitz smoothness, we have
\[
\begin{aligned}
&f(x_{k+1})\\
\leq&f(x_k)+\inner{\nabla f(x_k),x_{k+1}-x_k}+\frac{L}{2}\norm{x_{k+1}-x_k}^2\\
=&f(x_k)-\eta_k\inner{\nabla f(x_k),H_kg_k}+\frac{L\eta_k^2}{2}\norm{H_kg_k}^2\\
\leq&f(x_k)-\eta_k\inner{\nabla f(x_k),H_kg_k}+L\eta_k^2\norm{H_k\nabla f(x_k)}^2\\
&+L\eta_k^2\norm{H_k(\nabla f(x_k)-g_k)}^2.\\
\end{aligned}
\]
Rearranging the terms and taking expectations, we get
\[
\begin{aligned}
&\eta_k\E\left[\norm{\nabla f(x_k)}^2_{H_k}\right]\\
\leq&\E[f(x_k)]-\E[f(x_{k+1})]+L\eta_k^2\E\left[\norm{H_k\nabla f(x_k)}^2\right]\\
&+L\sigma^2\eta_k^2\norm{H_k}^2(1+\norm{\nabla f(x_k)}^2).
\end{aligned}
\]
By the choice of $c$, we have 
\[
\begin{aligned}
&\frac{1}{2}\eta_k\norm{\nabla f(x_k)}^2_{H_k}\\
\geq &L\eta_k^2\left(\norm{H_k\nabla f(x_k)}^2+\sigma^2\norm{H_k}^2\norm{\nabla f(x_k)}^2\right),
\end{aligned}
\] 
we have
\[
\begin{aligned}
&\frac{\sum_{k=0}^{T}\eta_k\E\left[\norm{\nabla f(x_k)}^2_{H_k}\right]}{2\sum_{k=0}^{T}\eta_k}\\
\leq&\frac{f(x_0)-\bar f+L\sigma^2\lambda_{H,\max}^2\sum_{k=0}^{T}\eta_k^2}{\sum_{k=0}^{T}\eta_k}.
\end{aligned}
\]
In particular, when $\eta_k=\frac{c}{\sqrt{T+1}}$, we have
\[
\E\left[\norm{\nabla f(\bar x_k)}^2\right]\leq\frac{2(f(x_0)-\bar f+c^2L\sigma^2\lambda_{H,\max}^2)}{c\lambda_{H,\min}\sqrt{T+1}}.
\]
\end{proof}

\subsection{Nonsmooth Nonconvex Training Loss} \label{app-nonsmooth}
Conventional techniques in stochastic optimization for nonsmooth nonconvex scenarios typically rely on the time-homogeneity of the associated dynamical system, as shown in \cite{benaim2005stochastic,davis2020stochastic}. Given a locally Lipschitz function $f$, by Rademacher's theorem, $f$ is differentiable almost everywhere. Thus, we have the following definition of subdifferential for a locally Lipschitz function.
\begin{definition}
The Clark subdifferential or subgradient \cite{clarke1990optimization} is defined as 
\[
\partial f(x):=\left\{
\begin{aligned}
&y:\;x_k\rightarrow x,\,\nabla f(x_k)\rightarrow y,\\
&\text{where $f$ is differentiable at $x_k$}
\end{aligned}
\right\}.
\]
\end{definition}

\begin{definition}
A locally Lipschitz function is $C^p$-Whitney stratifiable \cite{davis2020stochastic}, if the graph of $f$: ${\rm graph}(f):=\{(x,t):f(x)=t\}$ can be decomposed into finite $C^p$ manifolds, called strata, satisfying
\begin{enumerate}
    \item For any two strata $\calM_1$ and $\calM_2$, the following implication holds:
    \[    \calM_1\cap\overline{\calM_2}\neq\emptyset\;\Longrightarrow \calM_1\subset\overline{\calM_2}
    \]
    \item For any sequence of points \( z_k \) in a stratum \( \calM_1 \) converging to a point \( \bar{z} \) in a stratum \( \calM_2 \), if the corresponding normal vectors \( v_k \in N_{\calM_1}(z_k) \) converge to a vector \( v \), then the inclusion \( v \in N_{\calM_2}(\bar{z}) \) holds. Here $N_{\calM_i}$ is the normal space of $\calM_i$. 
\end{enumerate}
\end{definition}
For example, the function $-|x|$ is a $C^\infty$-Whitney stratifiable function, with its graph decomposable into the sets $\{(0,0)\}$, $\{(t,-t):t>0\}$ and $\{(t,t):t<0\}$.

\begin{theorem}
Suppose Assumption \ref{assumption:nonsmooth} holds, and assume the sequence \( \{x_k\} \) remains within a compact set. If the learning rate satisfies \( \sum_{k=1}^\infty \eta_k = \infty \) and \( \sum_{k=1}^\infty \eta_k^2 < \infty \), then
\[
\lim_{k \rightarrow \infty} {\rm dist}(x_k, \Omega) = 0,
\]
where \( \Omega := \{x : 0 \in \partial f(x)\} \) is the set of stationary points.
\end{theorem}

\begin{proof}
Define the interpolated process \( x(t) \) for \( \{x_k\} \) as follows:
\[
x(t) := x_k + \frac{t - t_{k-1}}{\eta_k} (x_{k+1} - x_k), \quad \text{for } t \in [t_{k-1}, t_k),
\]
where \( t_k := \eta_1 + \cdots + \eta_k \), \(t_0=0\). Define \( y(t) := H_k d_k \) for \( t \in [t_{k-1}, t_k) \), where $d_k\in\partial f(x_k)$. Thus, both \( x(t) \) and \( y(t) \) are piecewise linear functions. We also define time-shifted versions \( y^t(\cdot) := y(t + \cdot) \).

Let \( x_t(\cdot) \) denote the solution to the following ODE:
\[
\dot{x}_t(\tau) = -y(\tau), \quad x_t(t) = x(t), \quad \text{for any } \tau \geq t.
\]
By Assumption \ref{assumption:nonsmooth}, \( \sup_k \norm{d_k} \leq \ell \), so \( \sup_{t \geq 0} \norm{y(t)} \leq M \ell \). Therefore, the class of functions \( \{x_t(\cdot) : t \geq 0\} \) is uniformly equicontinuous. Using the assumptions on \( \{\xi_k\} \), the learning rate \( \{\eta_k\} \), and the boundedness of \( H_k \), it follows from \cite[Lemma A.1]{duchi2018stochastic} that for any \( T > 0 \),
\[
\lim_{t \rightarrow \infty} \sup_{\tau \in [t, t + T]} \norm{x(\tau) - x_t(\tau)} = 0.
\]

Since \( x(\cdot) \) is pointwise bounded, \( x_t(\cdot) \) is also pointwise bounded. By the Arzel\`{a}-Ascoli theorem, the equicontinuity of \( \{x_t(\cdot) : t \geq 0\} \) implies that it is relatively compact in the space of continuous functions, under the topology of uniform convergence over any compact set. The relative compactness of \( \{y^t(\cdot)\} \) can be similarly verified; see \cite{benaim2005stochastic, borkar2009stochastic} for further details on related functional analysis concepts.

For any fixed \( T > 0 \), by the definition of \( x_t(\cdot) \), we have
\[
x_t(t + T) = x_t(t) - \int_0^T y^t(s) \, ds.
\]
Now, select a subsequence \( \{t_{k_j}\} \) such that the sequences \( \{x_t(\cdot)\} \) and \( \{y^t(\cdot)\} \) converge to \(\bar x(\cdot)\) and \(\bar y(\cdot)\), respectively, as \( j \to \infty \). Letting \( j \to \infty \), we obtain
\[
\bar{x}(T) = \bar{x}(0) - \int_0^T \bar{y}(s) \, ds.
\]
Next, we show that $\bar{y}(s) \in \bar{H} \partial f(\bar{x}(s))$. Note that 
\[
\begin{aligned}
&\text{dist} \left( \bar{y}(s), \bar{H} \partial f(\bar{x}(s)) \right)\\ 
\leq& \left\| \frac{1}{N} \sum_{j=1}^N y^{t_{k_j}}(s) - \bar{y}(s) \right\|\\
&+ \text{dist} \left( \frac{1}{N} \sum_{j=1}^N y^{t_{k_j}}(s), \bar{H} \partial f(\bar{x}(s)) \right) \\
\leq& \text{dist} \left( \frac{1}{N} \sum_{j=1}^N H_{\lambda(t_{k_j} + s)} d_{\lambda(t_{k_j} + s)}, \bar{H} \partial f(\bar{x}(s)) \right) + o(1),
\end{aligned}
\]
where \( \lambda(t) = k \) such that \( t_k < t \leq t_{k+1} \). Since \( d_{\lambda(t_{k_j} + s)} \in \partial f(x_{\lambda(t_{k_j} + s)}) \), by the outer-semicontinuity of \( \partial f \), we have \(\dist\left(d_{\lambda(t_{k_j} + s)},\partial f(\bar{x}(s))\right)\to 0\). Using Assumption \ref{assumption:nonsmooth}c), we have
\[
\begin{aligned}
&\text{dist} \left( \bar{y}(s), \bar{H} \partial f(\bar{x}(s)) \right) \\
\leq &\text{dist} \left( \frac{1}{N} \sum_{j=1}^N H_{\lambda(t_{k_j} + s)} d_{\lambda(t_{k_j} + s)}, \bar{H} \partial f(\bar{x}(s)) \right) + o(1)\\
&\to 0.
\end{aligned}
\]
Thus, we conclude the following:
\begin{equation}
\label{eq:ode_nonsmooth}
\bar{x}(T) = \bar{x}(0) - \int_0^T \bar{y}(s) \, ds, \quad \text{and } \bar{y}(s) \in \bar{H} \partial f(\bar{x}(s)).
\end{equation}
By \cite[Theorem 3.2]{davis2020stochastic}, any limit point of \( \{x_k\} \) converges to the stable set of \eqref{eq:ode_nonsmooth}, namely, \( \{x : 0 \in \bar{H} \partial f(x)\} = \{x : 0 \in \partial f(x)\} = \Omega \). This completes the proof.
\end{proof}

\section{Experimental Details} \label{app-exp_detail}
\subsection{Toy Example} \label{app-toy}
Here we compare Cholesky quantization (CQ) and vanilla quantization (VQ) on a toy $2\times2$ matrix using 4-bit linear-2 quantization as introduced in \cref{sec-linear-2}. The original matrix, with eigenvalues $(10.908, 0.092)$, has a high condition number. VQ perturbs matrix elements, distorting the spectrum and producing a negative eigenvalue $-0.164$, breaking PD. In contrast, CQ quantizes the Cholesky factor, preserving structure and yielding eigenvalues $(11.310, 0.109)$, closer to the original. This shows CQ is \textit{more robust for ill-conditioned matrices}, mitigating instability and preserving spectral properties better than VQ.

\begin{table}[h]
\centering
\setlength{\tabcolsep}{1.5pt}  
\renewcommand{\arraystretch}{0.5}  
\small  
\caption{Comparison of VQ versus CQ on a toy $2\times 2$ matrix $L$.}
\label{tab:toy_example}
\begin{tabular}{l|ccc}
\toprule
Method & Original & VQ & CQ \\
\midrule
$L$ & 
$\begin{bmatrix} 
10.00 & 3.00 \\ 
3.00 & 1.00 
\end{bmatrix}$ 
& 
$\begin{bmatrix} 
10.00 & 3.60 \\ 
3.60 & 1.11 
\end{bmatrix}$ 
& 
$\begin{bmatrix} 
10.00 & 3.60 \\ 
3.60 & 1.42 
\end{bmatrix}$ 
\\
\midrule
Eigenvalues & $(10.908,  0.092)$ & $(11.275, -0.164)$ & $(11.310,  0.109)$ \\
\bottomrule
\end{tabular}
\end{table}

\subsection{Matrix Distance} \label{app-mat}
For the Frobenius norm relative error (NRE) and angle error (AE) in \cref{tab:nre_ae}, we report the cumulative errors over all preconditioners. For synthetic matrices, we randomly generate 100 instances of \( A \) via spectral decomposition to assess quantization robustness. Specifically, we construct \( A \) as:
\[
\setlength{\abovedisplayskip}{3pt}
\setlength{\belowdisplayskip}{3pt}
\setlength{\abovedisplayshortskip}{3pt}
\setlength{\belowdisplayshortskip}{3pt}
A = U \Lambda U^\top,
\]
where \( U \) is a randomly sampled orthogonal matrix obtained via QR decomposition of a Gaussian random matrix, and \( \Lambda \) is a diagonal matrix with eigenvalues geometrically spaced from \( 10^{-3} \) to \( 10^3 \). This setup ensures a high dynamic range, making small values more susceptible to quantization errors, which are further amplified during inverse 1/4-th root computations.

Additionally, we evaluate NRE and AE on preconditioners from 32-bit Shampoo training of Swin-Tiny on CIFAR-100. The results, summarized in \cref{apptab:nre_ae}, show that Cholesky quantization consistently reduces both NRE and AE compared to vanilla quantization, demonstrating its effectiveness in preserving spectral properties.

\begin{table}[ht]
    \centering
    \caption{NRE and AE on preconditioners of Swin-Tiny for vanilla quantization (VQ) and Cholesky quantization (CQ).}
    \label{apptab:nre_ae}
    \begin{small}
    \renewcommand{\arraystretch}{0.85}
    \setlength{\aboverulesep}{1pt}
    \begin{tabular}{l|cc|cc}
        \toprule
        Quantization & \multicolumn{2}{c|}{VQ} & \multicolumn{2}{c}{CQ} \\
        & NRE & AE & NRE & AE \\
        \midrule
        Epoch 25  & 36.669 & 29.669 & 9.381  & 9.344  \\
        Epoch 50 & 36.853 & 29.269 & 8.803  & 8.775  \\
        Epoch 75 & 39.494 & 30.686 & 8.814  & 8.804  \\
        Epoch 100 & 41.068 & 30.848 & 8.943  & 8.918  \\
        \bottomrule
    \end{tabular}
    \end{small}
\end{table}

\subsection{Training Hyperparameters} \label{app-exp_hyp}
For the first-order base optimizers SGDM and AdamW used in Shampoo, we maintain their optimizer states at the same precision as the model parameters, which is float32 for image classification and bfloat16 for LLM pre-training.

For SGDM, we set the initial learning rate to 0.1, the momentum parameter to 0.9, and the weight decay coefficient to $5 \times 10^{-4}$ for training CNNs on CIFAR-100 and Tiny-ImageNet, and $1 \times 10^{-4}$ for training ResNet-50 on ImageNet. For AdamW, we set the initial learning rate to $1 \times 10^{-3}$, the momentum parameters to $\beta_1 = 0.9$ and $\beta_2 = 0.999$, the small positive constant for the denominator to $1 \times 10^{-8}$, and the weight decay to $5 \times 10^{-2}$ for image classification and $0$ for LLM pre-training.

For quantization settings, we employ block-wise linear-2 quantization as introduced in \cref{sec-linear-2}, with a block size of $B \times B = 64 \times 64$. For tensors with fewer than 4096 elements, quantization is not applied.

For both 32-bit and 4-bit Shampoo, we set the small positive constant $\epsilon = 1 \times 10^{-6}$ and the preconditioner momentum parameter $\beta = 0.95$. The error state momentum parameter is set to $\beta_e = 0.95$ to align with the preconditioner update. For update intervals, we use $T_1 = 100$ and $T_2 = 500$ for experiments on CIFAR-100 and Tiny-ImageNet, $T_1 = 200$ and $T_2 = 1000$ for training ResNet-50 on ImageNet, and $T_1 =T_2 = 200$ for LLM pre-training. Additionally, Shampoo applies layer-wise preconditioning to blocks derived from large matrices, with the maximum order of the preconditioner set to 1200.

For image classification tasks, we use the traditional cross-entropy loss as the training loss. For the learning rate schedule, we employ cosine annealing with 5 epochs of linear warmup across all experiments. For data augmentation, we apply horizontal flip, random crop, and color jitter for VGG and ResNet \cite{krizhevsky2012imagenet,he2016deep}, and Mixup \cite{zhang2018mixup}, CutMix \cite{yun2019cutmix}, RandomErasing \cite{zhong2020random}, and RandAugment/AutoAugment \cite{cubuk2020randaugment,cubuk2018autoaugment} for Swin and ViT \cite{lee2021vision,liu2021swin}.

The batch size is set to 128 for experiments on CIFAR-100 and Tiny-ImageNet, 256 for training ResNet-50 on ImageNet, and 512 for training ViT-Base on ImageNet. For the total training epochs, we follow \cite{he2016deep,xie2024adan} and train Shampoo with SGDM as the base optimizer for 200 epochs when training CNNs on CIFAR-100, while SGDM itself is trained for 300 epochs on CIFAR-100. For training CNNs on Tiny-ImageNet and ViTs on CIFAR-100 and Tiny-ImageNet, we follow \cite{liu2021swin,lee2021vision} and train Shampoo with the base optimizer for 100 epochs, and the base optimizer itself for 150 epochs. For training ResNet-50 on ImageNet, we train Shampoo with SGDM as the base optimizer for 100 epochs and SGDM for 120 epochs. For training ViT-Base on ImageNet, we train Shampoo with AdamW as the base optimizer for 120 epochs and AdamW for 150 epochs.

For LLM pre-training, we follow the model settings of \cite{lialin2023relora,zhao2024galore}, with details provided in \cref{apptab:llama_hyp}. All experiments use bfloat16 to reduce memory usage. Due to limited computational resources, we shorten training and run 10K steps for LLaMA-130M and LLaMA-350M, and 2K steps for LLaMA-1B. The total effective batch size per training step is 512 with gradient accumulation. The per-iteration batch size is set to 256 for  LLaMA-130M, 128 for and LLaMA-350M, and 64 for LLaMA-1B.

\begin{table}[t]
\centering
\caption{Hyperparameters of LLaMA models for evaluation. Data amount are specified in tokens.}
\label{apptab:llama_hyp}
\small
\begin{tabular}{ccccc}
\toprule
Params & Hidden & Intermediate & Heads & Layers \\
\midrule
130M & 768  & 2048  & 12 & 12 \\
350M & 1024 & 2736  & 16 & 24 \\
1 B  & 2048 & 5461  & 24 & 32  \\
\bottomrule
\end{tabular}
\end{table}

\subsection{Memory Efficiency} \label{app-mem}
In our experiments, we report the peak GPU memory usage instead of the memory used solely by the optimizers, as the peak GPU memory usage is the primary constraint when training large-scale models in practice and is therefore our main concern. Furthermore, from the total peak GPU memory usage, we can deduce the additional memory cost introduced by the preconditioners of Shampoo on top of the base optimizers. 

For instance, when training ResNet-34 on CIFAR-100, the base optimizer SGDM incurs a peak memory usage of 1254.7 MB. The additional peak GPU memory usage caused by storing the 32-bit preconditioners of Shampoo $(L_k, R_k, L_k^{-1/4}, R_k^{-1/4})$ is calculated as the peak memory usage of 32-bit Shampoo minus 1254.7 MB, which equals 627.9 MB. With vanilla 4-bit quantization for the preconditioners, this additional memory usage drops to 86.3 MB, which is less than $1/7$ of the additional memory required by 32-bit Shampoo. Furthermore, when using 4-bit Shampoo with Cholesky quantization, the additional peak memory usage decreases further to 64.8 MB.

We now provide a brief analysis of why the increased peak memory usage of 4-bit Shampoo with Cholesky quantization (e.g., 64.8 MB) is approximately 75\% of that of vanilla 4-bit Shampoo (e.g., 86.3 MB). Vanilla 4-bit Shampoo stores the 4-bit preconditioners $(L_k, R_k, L_k^{-1/4}, R_k^{-1/4})$, as introduced in \cref{sec-vanilla}, which consist of four full matrices of the same shape in 4-bit precision. In contrast, 4-bit Shampoo with Cholesky quantization stores $(C_k^L, C_k^R, L_k^{-1/4}, R_k^{-1/4})$ as described in \cref{sec-cholesky}, where $C_k^L$ and $C_k^R$ are the lower triangular Cholesky factors of $L_k$ and $R_k$, respectively. The storage of $C_k^L$ and $C_k^R$ requires only half the space of $L_k$ and $R_k$, leading to the total storage cost of the preconditioners for 4-bit Shampoo with Cholesky quantization being approximately 75\% of that of vanilla 4-bit Shampoo.

For $L_k^{-1/4}$ and $R_k^{-1/4}$, Cholesky quantization is not applied, as they are used to precondition stochastic gradients at each iteration, as described in \cref{alg:32bit} and \cref{alg:4bit}. Restoring them from their Cholesky factors at each iteration would be computationally expensive.

\end{document}